\documentclass[runningheads]{llncs}
\usepackage[english]{babel}

\usepackage{graphicx}
\usepackage{microtype}
\usepackage{booktabs}

\usepackage{float}
\usepackage{wrapfig}
\usepackage[normalem]{ulem}

\usepackage{bm}
\usepackage{textcomp}
\usepackage{amssymb}
\usepackage{amsmath}
\usepackage{amsfonts}
\usepackage{mathtools}
\usepackage{pifont}

\usepackage{algorithm}
\usepackage{algorithmic}

\usepackage{times}
\usepackage{gensymb}

\usepackage{textcomp}
\usepackage{multirow}
\usepackage{lscape}
\usepackage{subcaption}

\usepackage{mdframed}
\usepackage{hyperref}
\usepackage{textgreek}
\usepackage{tabularx, colortbl}
\usepackage[dvipsnames]{xcolor}
\definecolor{mygray}{gray}{0.95}
\definecolor{mycyan}{HTML}{005397}
\definecolor{myred}{HTML}{E13333}
\definecolor{mymagenta}{HTML}{BF3E87}
\definecolor{mypurple}{HTML}{1B2278}

\definecolor{tearose}{HTML}{F584C5}
\definecolor{coral}{HTML}{F67088}
\definecolor{dodger_blue}{HTML}{3BA3EC}
\definecolor{domino}{HTML}{BC9F48}
\definecolor{domino}{HTML}{BC9F48}
\definecolor{domino}{HTML}{BC9F48}
\definecolor{catalina_blue}{HTML}{1C3168}
\definecolor{catalina_blue}{HTML}{1C3168}
\definecolor{catalina_blue}{HTML}{1C3168}
\definecolor{dark_scarlet}{HTML}{C63D52}
\definecolor{cerulean}{HTML}{0192A8}

\definecolor{tussock}{HTML}{C99E31}
\definecolor{p13}{HTML}{BFB5D7}
\definecolor{b14}{HTML}{BEA1A5}
\definecolor{y15}{HTML}{F0Cf61}
\definecolor{Merino}{HTML}{F3EEE3}

\newcolumntype{a}{>{\columncolor{p13}}l}

\usepackage[capitalize,noabbrev]{cleveref}
\crefname{ineq}{Inequality}{Inequalities}
\creflabelformat{ineq}{#2{\upshape(#1)}#3}

\usepackage{braket}

\newcommand{\abs}[1]{\left\lvert#1\right\rvert}

\usepackage{stmaryrd}

\usepackage{tikz}
\usetikzlibrary{trees}
\usepackage{tikz-dependency}
\usepackage{tikz-3dplot}
\usetikzlibrary{chains,scopes}
\usetikzlibrary{arrows.meta,automata,positioning}
\usetikzlibrary{shapes.geometric, arrows}
\usepackage{pgfplots}

\pgfplotsset{
  every axis/.append style = {thick},
  tick style = {thick,black},
  /tikz/normal shift/.code 2 args = {%
    \pgftransformshift{%
        \pgfpointscale{#2}{\pgfplotspointouternormalvectorofticklabelaxis{#1}}%
    }%
  },%
  shift/.style = {
    tick align        = outside,
    scaled ticks      = false,
    enlargelimits     = false,
    ticklabel shift   = {#1},
    axis lines*       = left,
    xtick style       = {normal shift={x}{#1}},
    ytick style       = {normal shift={y}{#1}},
    x axis line style = {normal shift={x}{#1}},
    y axis line style = {normal shift={y}{#1}},
  },
  shift/.default = 10pt,
  shift3d/.style = {
    shift=#1,
    ztick style       = {normal shift={z}{#1}},
    z axis line style = {normal shift={z}{#1}},
  },
  shift3d/.default = 10pt,
}

\usepackage{listings}
\usepackage{array}
\newcolumntype{H}{>{\setbox0=\hbox\bgroup}c<{\egroup}@{}}

\begin{document}

\title{Certified Robust Accuracy of Neural Networks Are Bounded due to Bayes Errors}

\author{Ruihan Zhang\and Jun Sun}

\authorrunning{R. Zhang \and J. Sun}

\institute{Singapore Management University, Singapore, Singapore\\
\email{rhzhang@smu.edu.sg}}

\maketitle

\begin{abstract}
    Adversarial examples pose a security threat to many critical systems built on neural networks. While certified training improves robustness, it also decreases accuracy noticeably. Despite various proposals for addressing this issue, the significant accuracy drop remains. More importantly, it is not clear whether there is a certain fundamental limit on achieving robustness whilst maintaining accuracy. In this work, we offer a novel perspective based on Bayes errors. By adopting Bayes error to robustness analysis, we investigate the limit of certified robust accuracy, taking into account data distribution uncertainties. We first show that the accuracy inevitably decreases in the pursuit of robustness due to changed Bayes error in the altered data distribution. Subsequently, we establish an upper bound for certified robust accuracy, considering the distribution of individual classes and their boundaries. Our theoretical results are empirically evaluated on real-world datasets and are shown to be consistent with the limited success of existing certified training results, \emph{e.g.}, for CIFAR10, our analysis results in an upper bound (of certified robust accuracy) of 67.49\%, meanwhile existing approaches are only able to increase it from 53.89\% in 2017 to 62.84\% in 2023.

\end{abstract}

\section{Introduction}
Neural networks have achieved remarkable success in various applications, including many security-critical systems such as self-driving cars~\cite{kurakin2018adversarial}, and face-recognition-based authentication systems~\cite{sharif2016accessorize}. Unfortunately, several security issues of neural networks have been discovered as well. Arguably the most notable one is the presence of adversarial examples. Adversarial examples are inputs that are carefully crafted by adding human imperceptible perturbation to normal inputs to trigger wrong predictions~\cite{kurakin2016adversarial}. Their existence poses a significant threat when the neural networks are deployed in security-critical scenarios. For example, adversarial examples can mislead road sign recognition systems of self-driving cars and cause accidents~\cite {kurakin2018adversarial}. 
The increasing adoption of machine learning in security-sensitive domains raises concerns about the robustness of these models against adversarial examples~\cite{papernot2016transferability}.

To defend against adversarial examples, various methods for improving a model's robustness have been proposed. Two main categories are adversarial training~\cite{bai2021recent,wong2020fast} and certified training~\cite{muller2022certified,shi2021fast}, both of which aim to improve a model's accuracy in the presence of adversarial examples. Adversarial training works by training the network with a mix of normal and adversarial examples, either pre-generated or generated during training. Methods in this category do not provide a formal robustness guarantee~\cite{zhang2019limitations}, leaving the system potentially vulnerable to new types of adversarial attacks~\cite{liu2019adaptiveface,tramer2020adaptive}.

In contrast, certified training aims to provide a formal guarantee of robustness. A method in this category typically incorporates robustness verification techniques~\cite{xu2020automatic} during training, \emph{i.e.}, they aim to find a valuation of network parameters such that the model is provably robust with respect to the training samples. However, they are found to reduce the model's accuracy significantly~\cite{chiang2020certified}. Recent studies have shown that state-of-the-art certified training can result in up to 60\% accuracy drop on CIFAR-10~\cite{lyu2021towards} (at vicinity size 8/255). This is unacceptable for many real-world applications. Although numerous researchers attempt to enhance certified training methods, there seems to be an invisible hurdle preventing them from achieving a level of accuracy similar to that of vanilla models. Despite attempts to explore it using the limit of certain abstraction domains~\cite{mirman2021fundamental}, in general, whether there is such a theoretical upper bound on certified robust accuracy or not remains an open problem.

In this work, we offer a novel perspective and argue that Bayes errors may be one of the reasons why there is such an invisible hurdle.  The Bayes error, in the context of statistics and machine learning, is a fundamental concept related to the inherent uncertainty in any classification system~\cite{ishida2022performance}. It represents the minimum error rate for any classifier on a given problem and is determined by the overlap in the probability distributions of the classes to be predicted~\cite{fukunaga1990introduction}. Thus, we study whether the Bayes errors put a limit on certified robust accuracy.
 
To understand how Bayes Error is relevant, we can consider it from the uncertainty in neural network learning. Most existing classifiers learn using a data set which gives a unique and certain label for each input~\cite{lecunmnist}. Yet, this may not be the case in reality. That is, not every input may have a 100\% certain label (due to reasons such as information loss during the picture-capturing process). Intuitively, we show a real-world example in \cref{fig:dog}. This image looks like a cat, while it is, in fact, also possible to be a dog. The point is that unless we know how this photo was taken, and there is no information loss during the photo taking, there may always be a certain level of uncertainty when we label the data. These uncertainties call for Bayes errors and actually bounds both vanilla and certified robust accuracy.

\begin{figure}[t]
    \centering
    \begin{subfigure}[t]{0.25\linewidth}
        \includegraphics[width=\linewidth]{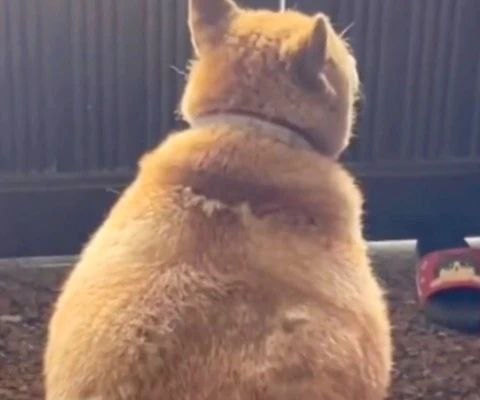}
        \caption{Is this a cat?}
        \label{fig:dogback}
    \end{subfigure}
    \hspace{1pt} 
    \begin{subfigure}[t]{0.25\linewidth}
        \includegraphics[width=\linewidth]{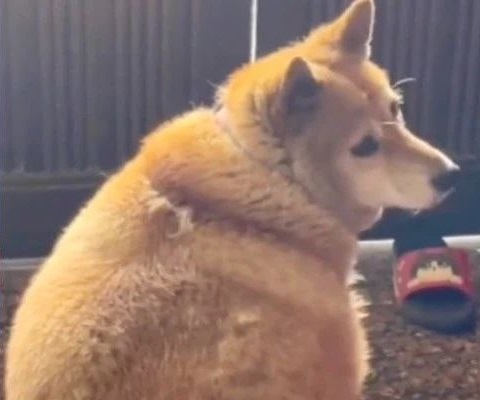}
        \caption{Eh, it's a dog.}
        \label{fig:dogside}
    \end{subfigure}
    \caption{The picture at left may look like a cat. In fact, it can be the back of a dog. }
    \label{fig:dog}
\end{figure}

This work has two objectives. First, we aim to analyse whether the quest for robustness inevitably decreases model accuracy, from the perspective of Bayes errors. This requires examining how the inherent, irreducible error in class probability distributions influences the robustness of classifiers (with respect to perturbations). We show that given the definition of robustness, the data distribution undergoes a convolution within the vicinity (\emph{i.e.}, the region around an input which is defined by the perturbation budget). Second, we intend to quantify this potential decrease, \emph{i.e.}, what are the upper bounds on the optimal certified robust accuracy? We show that such an upper bound can be derived independently from the classification algorithm. Through a detailed exploration of how each input may contribute to the Bayes errors, our study aims to enhance the understanding of their contribution to classification robustness.

We apply our analysis to multiple benchmark data sets and the corresponding models. From every data set, we observe that the convolved distribution has an increased Bayes error compared with the original distribution. This implies that pursuing robustness would in turn increase uncertainty, and decrease accuracy as we show in our analysis. Second, we contrast the state-of-the-art (SOTA) certified robust accuracy against the upper bound derived using our approach. This is to verify if the bound is empirically effective. We find that the bound is indeed higher than the state-of-the-art certified robust accuracy. We further investigate the relationship between the robustness upper bound and the perturbation vicinity size. When vicinity size grows, we expectedly obtain a decreased upper bound, on every data set used in our study.

\section{Preliminary and Problem Definition}
\label{sec:background}

In this section, we review the relevant background of this study, including the fundamentals of robustness in machine learning, \emph{e.g.}, its definition and verification. We also recall statistical decision theory, highlighting its relevance to classification. After that, we define our research problem.

In machine learning, the learner, denoted as a function $h: \mathbb{X} \to \mathbb{Y}$, is used to predict outputs $h(\bm{x})\in \mathbb{Y}$ based on a (possibly high dimensional) input point $\bm{x}\in\mathbb{X}$. The quality of $h$ can be measured by a problem-dependent loss function $\ell(h, \mathbf{x}, \textnormal{y})$~\cite{zhang2004statistical}. The choice of the loss function depends on the specific problem and data. Common options include the cross-entropy loss for classification and the mean squared error loss for regression. We focus on the classification problem in this work. Classification is the problem of assigning a class to each input~\cite{mohri2018foundations}, \emph{i.e.}, the learner’s task is to map an input to a discrete class and the learner is often called a classifier.
\begin{definition}[Classifier]
   In machine learning, a classifier maps an input $ \bm{x} $ from an input space $ \mathbb{X} $ to a discrete class $ y $ in the output space $ \mathbb{Y} $. The output space $ \mathbb{Y} $ is a (typically finite) set of discrete categories. Formally,
   \begin{equation}
       h:\bm{x}\mapsto\hat{y}, \quad  \bm{x} \in \mathbb{X}, \, \hat{y} \in \mathbb{Y}, \quad \mathbb{Y} = \set{\mathrm{class}_i\mid i\in\mathbb{Z}^+, i\le N_y}
   \end{equation}
where $N_y$ is the total number of categories the classifier can assign such that $\abs{\mathbb{Y}} = N_y$.
\end{definition}

For example, a spam classifier maps an email to $\set{\text{spam}, \text{non-spam}}$. The input vector for an email may embody the length of the message, the frequency of certain keywords in the body of the message, or the vectorised email body~\cite{mohri2018foundations}. A learning example contains an input and a label. A classifier can learn from labelled email examples and predict labels for other email examples. The classifier's predictions are then compared with the labels of the email under test to measure the performance of the classifier, \emph{e.g.}, a zero-one misclassification loss may be defined over $\set{\text{spam}, \text{non-spam}}\times\set{\text{spam}, \text{non-spam}}$ by $l(\hat{y}, y) = \bm{1}_{\hat{y}\neq y}$. A lower loss on the test sample set indicates a more accurate classifier.

\subsection{Robustness in Classification}
A classifier may not be robust as small changes in input data might lead to significant changes in the predictions made by a classifier~\cite{szegedy2013intriguing}. Consider the spam classifier example. Surprisingly, the removal of a single seemingly unimportant word from an email may switch the classifier's decision from spam to non-spam~\cite{wang2021crafting}. This phenomenon highlights the existence of adversarial examples, which are defined as follows.
\begin{definition}[Adversarial example~\cite{goodfellow2014explaining,kurakin2016adversarial}]
\label{def:ae}
    Given a classifier $h:\mathbb{X}\to\mathbb{Y}$ and an input-label pair $ (\bm{x}, y) \in \mathbb{X}\times\mathbb{Y} $, an adversarial example $ \bm{x'}\in\mathbb{X} $ is an input that is similar to $\bm{x}$ but is classified wrongly, \emph{e.g.}, $h(\bm{x'}) \neq y$. The difference between $ \bm{x} $ and $ \bm{x'} $ can be measured by a distance function $ d $, and we often require the distance between $\bm{x'}$ and $\bm{x}$ to be smaller than some threshold $\epsilon$. We assume $\bm{x}$ is correctly classified, \emph{i.e.}, $h(\bm{x}) = y$.
\end{definition}
Consider the case of the spam email. If a single word is removed, the Levenshtein distance (a measure of the number of edits needed to change one text into another) is 1. An adversarial example based on such a small change could be used with malicious intent. Even though removing a common word like `just' does not alter the nature of a spam email, it might be enough to prevent it from being detected by the spam classifier. Therefore, robustness against such attacks is needed such that spam would not evade detection by just changing a few words. Formally, robustness is defined as follows~\cite{lin2019robustness}.
\begin{definition}[Classifier robustness against perturbations]
\label{def:robustness}
    Given classifier $h$ and example $(\bm{x}, y)\in\mathbb{X}\times\mathbb{Y}$, we say that $h$ is robust with respect to vicinity $\set{\bm{x'}\mid d(\bm{x}, \bm{x'}) \le \epsilon}$, \emph{i.e.}, $\operatorname{Rob}\big(h, \bm{x}, y; (d, \epsilon)\big)$, only if the following condition is satisfied.
    \begin{equation}
    \label{eq:blind}
        \lnot \exists~ \bm{x'}\in \mathbb{X}.~ d(\bm{x}, \bm{x'}) \le \epsilon \land  h(\bm{x'})\neq y
    \end{equation}
\end{definition}
\cref{def:robustness} involves the concept of vicinity, which is a subset of the input space, \emph{i.e.} $\subset \mathbb{X}$. It is usually determined by an input and a budget for perturbation. For instance, give an input $\bm{x}$, we can define its vicinity as $\mathbb{V}_{\bm{x}} = \set{\bm{x'}| d(\bm{x}, \bm{x'}) \le \epsilon}$. However, this set representation may be inconvenient sometimes. Thus, we give an equivalent function form as follows.
    \begin{equation}
    \label{eq:vicinity}
        v_{\bm{x}}(\bm{x'}) = \begin{cases}
            \left(\int_{\mathbb{V}_{\bm{x}}}d\bm{x''}\right)^{-1}, &\text{if}~~ \bm{x'}\in \mathbb{V}_{\bm{x}}\\
            0, &\text{otherwise}
        \end{cases}
    \end{equation}
    Essentially, \cref{eq:vicinity} can be viewed as a probability density function uniformly defined over the vicinity around an input $\bm{x}$. Now we shift the x-coordinate by $\bm{x}$, we get 
    \begin{equation}
        v_{\bm{0}}(\bm{x'} - \bm{x}) = \begin{cases}
            \left(\int_{\mathbb{V}_{\bm{0}}}d\bm{x''}\right)^{-1}, &\text{if}~~ \bm{x'} - \bm{x}\in \mathbb{V}_{\bm{0}}\\
            0, &\text{otherwise}
        \end{cases}
    \end{equation}
    Assuming that the vicinity function is translation invariant, we can drop the subscript $\bm{0}$, and use a positive constant $\epsilon_\mathrm{v}$ to represent $\int_{\mathbb{V}_{\bm{0}}}d\bm{x''}$. Thus, the vicinity function $v: \mathbb{X}\to \set{0, \epsilon_\mathrm{v}^{-1}}$ can be expressed as
    \begin{equation}
        v(\bm{x}) = \begin{cases}
            \epsilon_\mathrm{v}^{-1} &\text{if}~~ \bm{x} \in \mathbb{V}_{\bm{0}}\\
            0, &\text{otherwise}
            \end{cases}
    \end{equation}
    Since these representations are equivalent, we choose either representation based on the contexts. An example of a one-dimensional input's vicinity is shown in \cref{fig:vicinity1}.

\begin{figure}[t]
    \centering
    \begin{subfigure}[t]{0.386\linewidth}
        \includegraphics[width=\linewidth]{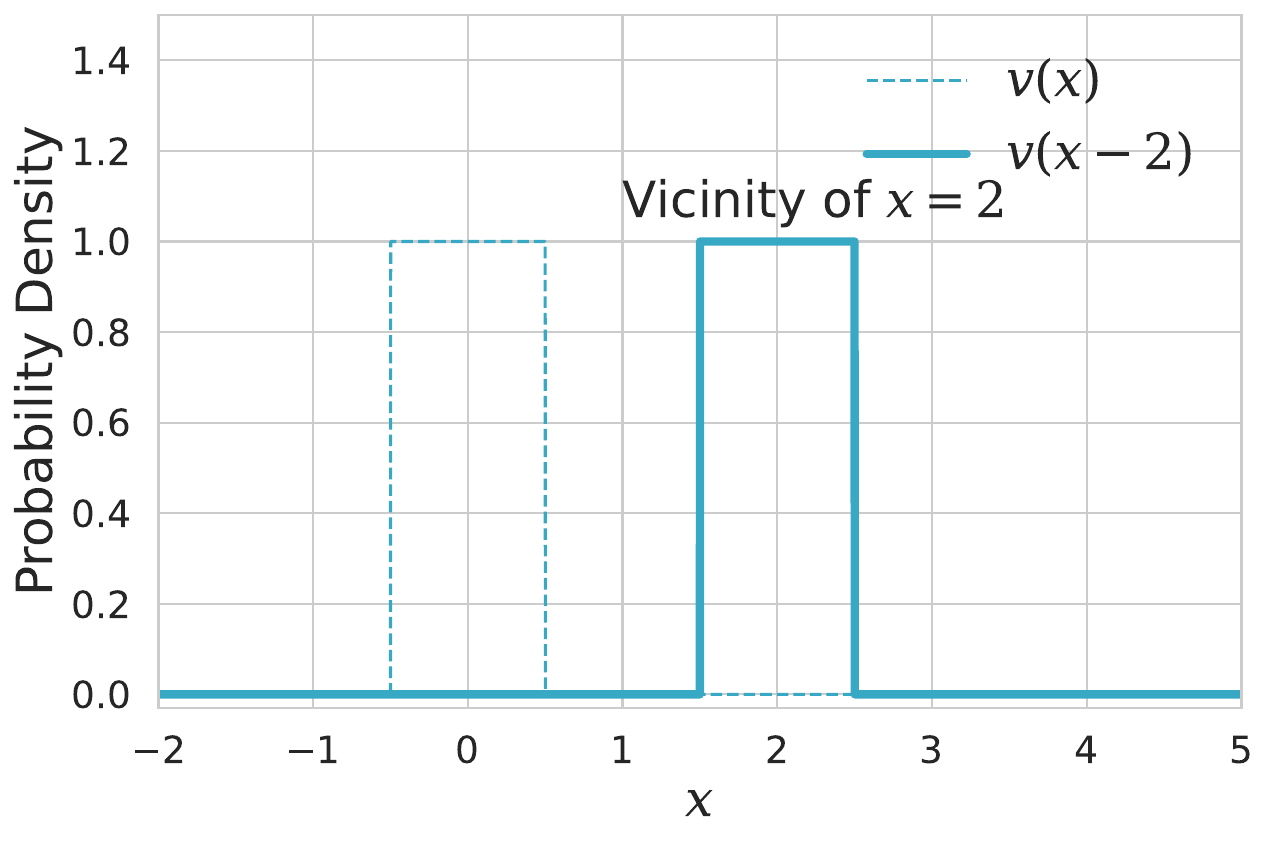}
        \caption{Vicinity function $v(\bm{x})$ is shown in blue ($\epsilon=0.5$). To get the vicinity at a specific input $\bm{x}=2$, we simply shift $v(\bm{x})$ along the positive direction of the x-axis by 2.}
        \label{fig:vicinity1}
    \end{subfigure}
    \hfill
    \begin{subfigure}[t]{0.59\linewidth}
        \includegraphics[width=\linewidth]{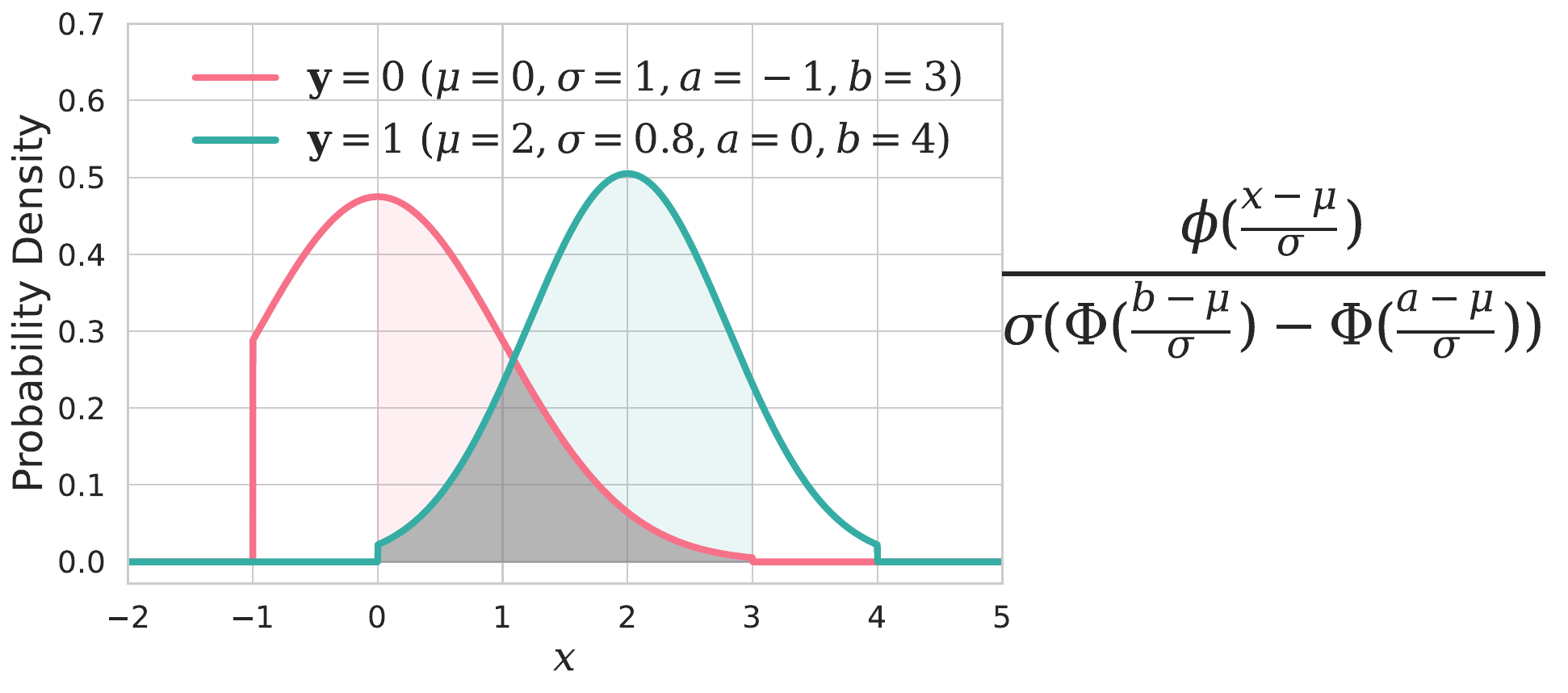}
        \caption{The grey-shaded area represents the Bayes error, characterising the overlap between the two distributions (truncated normal as expressed). The union of the light colour-shaded area and the grey-shaded area represents the proportion of inputs that \textit{have} uncertainty, \textit{i.e.}, $p(\mathbf{x}\in\mathbb{K}_\mathrm{D})$.}
        \label{fig:runex1}
    \end{subfigure}
    \caption{1D visualizations of vicinity function and Bayes error. A vicinity function is a rectangular function that returns a constant value if an input is in the vicinity. We use two PDFs of the truncated normal distribution to visualise the Bayes error.}
    \label{fig:pre}
\end{figure}

Achieving robustness is challenging. Verifying whether $\operatorname{Rob}\big(h, \bm{x}, y; \mathbb{V}_{\bm{x}}\big)$ holds for a given classifier $h$ is complicated since examining every example within a vicinity is impractical. Consequently, accurately estimating a classifier's robustness on specific inputs, as well as its robustness on a given data distribution, presents significant challenges. Existing methods for evaluating robustness include empirical evaluation (\emph{i.e.}, adversarial attacks)~\cite{ganin2016domain}, robustness verification~\cite{goodfellow2014explaining,li2023sok}, and others~\cite{weng2018evaluating}.

Adversarial attacks take one or more steps to search for adversarial examples within a vicinity. Let $\operatorname{AttS}\big(h, \bm{x}, y; \mathbb{V}{\bm{x}}\big)$ denote the \textbf{s}uccess of an \textbf{att}ack in finding adversarial examples in $\mathbb{V}_{\bm{x}}$. The failure rate of this attack on classifier $h$ can serve as an estimation for the classifier's expected robustness, as outlined below.
\begin{equation}
    \operatorname{AttS}\big(h, \bm{x}, y; \mathbb{V}{\bm{x}}\big) \to \lnot\operatorname{Rob}\big(h, \bm{x}, y; \mathbb{V}_{\bm{x}}\big)
\end{equation}

Another perspective contends that any non-zero rate of false negatives in the detection of adversarial examples is problematic. To this end, a condition $\operatorname{Vrob}$ is to be established, such that given a classifier $h$, it satisfies
\begin{equation}
\label{eq:incplt}
    \forall~(\bm{x}, y) \in \mathbb{X}\times\mathbb{Y}.~\operatorname{Vrob}\big(h, \bm{x}, y; \mathbb{V}_{\bm{x}}\big) \to \mathrm{Formula}~\ref{eq:blind}
\end{equation}
This method refers to robustness verification and the condition $\operatorname{Vrob}$ is the \textbf{v}erification result of \textbf{rob}ustness. There are two categories of robustness verification methods, \emph{i.e.}, incomplete deterministic verification~\cite{li2023sok} and complete deterministic verification~\cite{li2023sok}. Any deterministic verification method that fulfils Formula~\ref{eq:incplt} qualifies as an incomplete verification. If a method further fulfils Formula~\ref{eq:cmplt}, it qualifies as a complete verification. In both cases, if the verification result $\operatorname{Vrob}\big(h, \bm{x}, y; \mathbb{V}_{\bm{x}}\big)$ is \texttt{True}, \emph{i.e.}, verified, the classifier is considered to have deterministic robustness certification~\cite{li2023sok} for input $\bm{x}$ within the vicinity $\mathbb{V}_{\bm{x}}$, and the average certification likelihood is often called certified robust accuracy~\cite{shi2021fast}. Certified robust accuracy can serve as a lower bound for the classifier's expected robustness.
\begin{equation}
\label{eq:cmplt}
    \forall~(\bm{x}, y) \in \mathbb{X}\times\mathbb{Y}.~\operatorname{Vrob}\big(h, \bm{x}, y; \mathbb{V}_{\bm{x}}\big) \gets \operatorname{Rob}\big(h, \bm{x}, y; \mathbb{V}_{\bm{x}}\big)
\end{equation}
These verification methods can be used to optimise classifiers during training, and such a practice refers to certified training~\cite{li2023sok,vaishnavi2022accelerating}, which is defined as follows. 
\begin{equation}
\label{eq:cert_train}
    \min_h ~\operatorname{E}_{(\mathbf{x}, \textnormal{y})\sim D} \left[ \sup_{\bm{x'}\in \mathbb{V}(\mathbf{x}),~k\neq \textnormal{y}}\Big(\ell( h, \bm{x'}, \textnormal{y}) - \ell( h, \bm{x'}, k \Big)\right]
\end{equation}
Here, the neural network verification methods are used to soundly approximate the worst loss that can be induced by any perturbation within the vicinity of each training sample. However, after years of research~\cite{zhang2018efficient,singh2019abstract,balunovic2020adversarial}, certified training still faces challenges. Existing certified training methods often result in a significant drop in the model's  accuracy~\cite{cohen2019certified,raghunathan2018certified}. For instance, the best accuracy achieved by certified training is typically half of that of the standard training on the CIFAR-10 data set~\cite{tsipras2018robustness,shi2021fast}. Such a significantly reduced accuracy often means that the model is unacceptable in practice.

In summary, to evaluate whether $h$ attains robustness at example $(\bm{x}, y)$ within the vicinity $\mathbb{V}_{\bm{x}}$, existing methods include checking $\operatorname{AttS}\big(h, \bm{x}, y; \mathbb{V}{\bm{x}}\big)$ through adversarial attacks or $\operatorname{V}\big(h, \bm{x}, y; \mathbb{V}_{\bm{x}}\big)$ through robustness verification.
The expected robustness over a given distribution $D$, denoted by
\begin{equation}
\label{eq:expectedrobust}
    \operatorname{E}_{(\mathbf{x}, \textnormal{y}) \sim D}\left[\bm{1}_{\operatorname{Rob}\big(h, \mathbf{x}, \textnormal{y}; \mathbb{V}_{\mathbf{x}}\big)}\right]
\end{equation}
which can be overestimated by attack success rate ($\operatorname{E}_{(\mathbf{x}, \textnormal{y}) \sim D}\left[\bm{1}_{\operatorname{AttS}}\right]$) or underestimated by certified robust accuracy ($\operatorname{E}_{(\mathbf{x}, \textnormal{y}) \sim D}\left[\bm{1}_{\operatorname{Vrob}}\right]$). $\bm{1}_\mathrm{condition}$ is the indicator function that returns 1 if the condition is \texttt{True}, and 0 otherwise.

\subsection{Bayes Rules for Distributions}
In the following, we introduce the notion of Bayes Error and how it reflects a classification distribution. We consider a scenario where an input $ \mathbf{x} $ is to be classified into one class in $ \mathbb{Y} $, in particular, $ \textnormal{y} = k$ with prior class probability $P(\textnormal{y}=k)$ where $k\in \mathbb{Y}$. Let $ p(\mathbf{x} | \textnormal{y}=k) $ denote the class likelihood, that is, the conditional probability density of $ \mathbf{x} $ given that it belongs to class $ k$. The probability that the input $ \mathbf{x} $ belongs to a specific class $ k $, namely the posterior probability $ p(\textnormal{y}=k | \mathbf{x}) $, is given by Bayes' theorem.
\begin{equation}
\label{eq:bayestheorem}
    p(\textnormal{y}=k | \mathbf{x}) = \frac{p(\mathbf{x} | \textnormal{y}=k)P(\textnormal{y}=k)}{p(\mathbf{x})}
\end{equation}
where $p(\mathbf{x})$ is the probability density function of $\mathbf{x}$, \emph{i.e.}, $p(\mathbf{x}) = \sum_{k\in\mathbb{Y}} p(\mathbf{x}|\textnormal{y}=k) P(\textnormal{y}=k)$.
This classifier assigns an input $\mathbf{x}$ to the class with the highest posterior and is called the Bayes classifier, which is the optimal classifier. The classification error associated with the Bayes classifier is outlined as follows.
\begin{definition}[Bayes error]
Given a distribution $D$ over $\mathbb{X}\times\mathbb{Y}$, the error associated with the Bayes classifier is called the Bayes error (rate), denoted as $\beta_D$. The Bayes error can be expressed~\cite{fukunaga1990introduction,garber1988bounds} as: 
\begin{equation}
\label{eq:bayeserror1}
\begin{aligned}
    \beta_D &= \operatorname{E}_{(\mathbf{x}, \textnormal{y}) \sim D}\left[1 - \max_k p(\textnormal{y}=k|\mathbf{x})\right]\\ &= \int \left(1 - \max_k p(\textnormal{y}=k|\mathbf{x}=\bm{x})\right) p(\bm{x}) d \bm{x}
\end{aligned}
\end{equation}
Besides, since the Bayes classifier is optimal~\cite{ripley1996pattern}, this optimality gives rise to the following definition of the Bayes error~\cite{mohri2018foundations}.
\begin{equation}
\label{eq:bayeserror2}
    \beta_D = \min_{\textnormal{measurable}~ h} \operatorname{E}_{(\mathbf{x}, \textnormal{y}) \sim D} \left[\bm{1}_{h(\mathbf{x}) \neq \textnormal{y}}\right]       
\end{equation}
where the Bayes error is defined as the minimum of the errors achieved by measurable functions $h: \mathbb{X} \to \mathbb{Y}$. Hereby, (any) classifier $h$ with an error rate equal to $\beta_D$ can be called a Bayes classifier. 
\end{definition}

An example illustrating the Bayes error is given in \cref{fig:runex1}. The Bayes error fundamentally reflects the inherent uncertainty in classification tasks. It is the (irreducible) minimal error rate achievable by any classifier for a specific problem, influenced by the overlap amount among the class probability distributions. An input having a certain (deterministic) label can be formally expressed as $\max_k p(\textnormal{y}=k|\mathbf{x} = \bm{x}) = 1$. We can also represent this using the ceiling $\lceil \cdot \rceil$ or floor $\lfloor \cdot \rfloor$ function within the interval $[0, 1]$. Specifically, the ceiling function returns the smallest integer greater than or equal to the input. Consequently, it returns 1 for any number from 0 (exclusive) up to 1 and returns 0 if the input is 0. This shows that the input's label has uncertainty if $1 - \max_k p(\textnormal{y}=k|\mathbf{x} = \bm{x}) > 0$ and does not have uncertainty otherwise. We write $\mathbb{K}_D = \set{\bm{x}|1 - \max_k p(\textnormal{y}=k|\mathbf{x} = \bm{x}) > 0}$ to denote the set of every input whose label has uncertainty. The Bayes error provides a yardstick for other classifiers~\cite{ripley1996pattern,hastie2009elements}, \emph{e.g.}, a classifier may be deemed effective if its error rate approximates the Bayes error.

As highlighted in \cref{eq:bayestheorem} and (\ref{eq:bayeserror1}), the calculation of Bayes error is contingent upon knowing the prior distribution. In practical situations, since this distribution is not analytically known, the strategy is to estimate Bayes error using the observable portion of the distribution, \emph{e.g.}, training data characteristics, through approximations~\cite{hummels1987nonparametric,tumer1996estimating,yang2012discriminative,duvcinskas2015actual} or by computing its upper~\cite{fukunaga1990introduction,kang2004product,balagani2007relationship} and lower~\cite{antos1999lower,zhang2007gene,zhao2013beyond} bounds. \\

\noindent \textbf{Problem Definition} Next, we define the problem that we study. Despite the many proposals on certified training, noticeable suboptimality in robustness persists, especially compared with vanilla accuracy. Our objective is to ascertain whether this limitation comes from insufficient optimisation, or if there exists a fundamental upper bound that inherently limits the certified robust accuracy. Furthermore, if such an upper bound does exist, we aim to investigate how we can compute it, and how we can validate our result.

\section{An Upper Bound of Robustness from Bayes Error} \label{sec:method}
In this section, we present a method that attempts to address our research problem defined above, from a Bayes error perspective. Particularly, we hypothesise that the Bayes error plays a vital role in estimating the robustness that can be achieved by any classifier. First, we prove that certified training increases the Bayes error, which poses an upper bound on the robustness that can be achieved by any classifier. Second, we present how the upper bounds of certified robust accuracy can be calculated from a given distribution.

\subsection{Certified training increases Bayes error}
\label{sec:conv}
Certified robustness can be viewed as a way of optimizing the classifier with an altered data distribution instead of the original distribution~\cite{pang2022robustness}. This is because due to the requirement of robustness, an input may be forced with a label of some of its neighbors in the vicinity, instead of its original label. In the following, we investigate how the robustness requirement influences the data distribution, further affecting the Bayes error. We hypothesise that the altered distribution worsens Bayes error. We begin by defining a ``label-assignment'' action that alters a distribution, from a local perspective.

\begin{figure}[t]
    \centering
    \begin{subfigure}[t]{0.3\linewidth}
        \includegraphics[width=\linewidth]{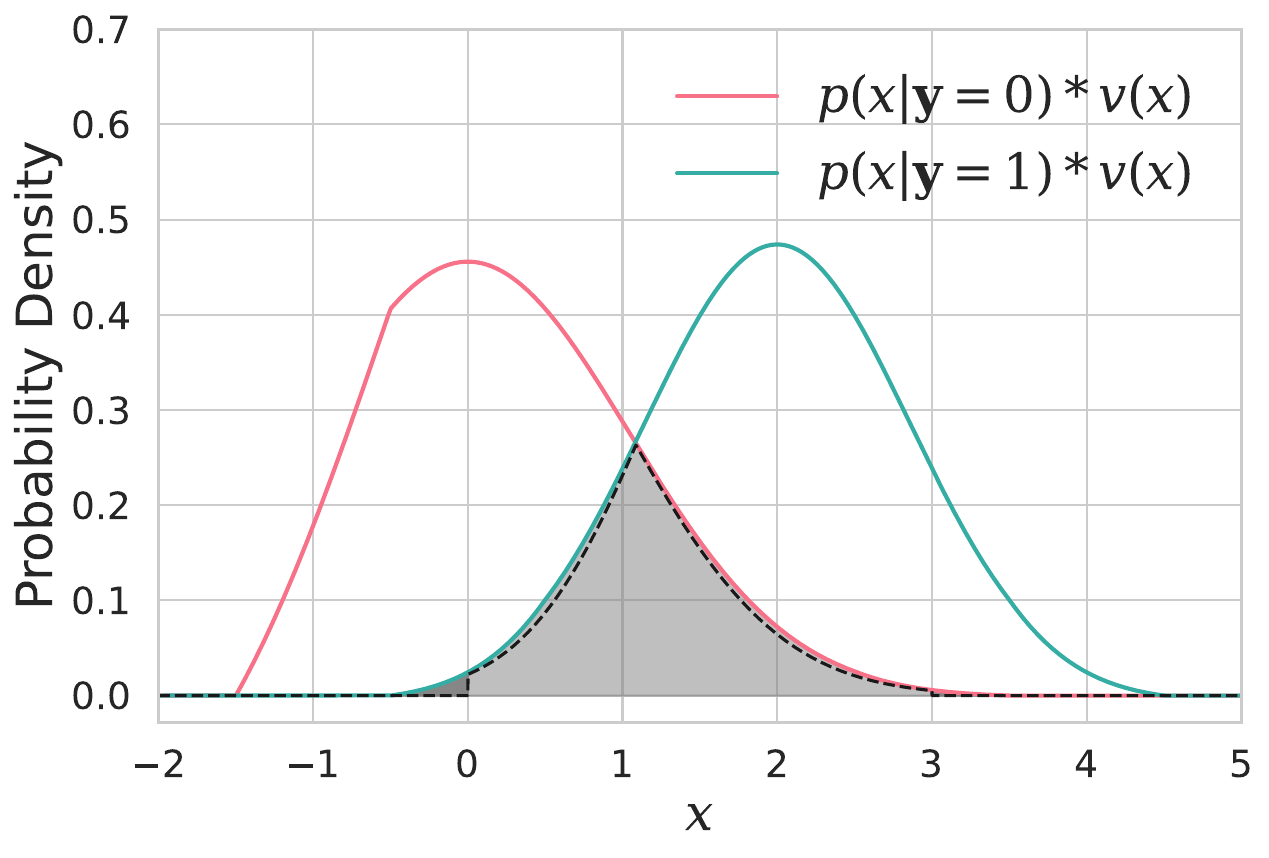}
        \caption{Convolution of distributions from Fig. \ref{fig:runex1} with the vicinity function from Fig. \ref{fig:vicinity1}, resulting in flatter and smoother distributions with increased shaded areas.}
        \label{fig:runex_conv}
    \end{subfigure}
    \hfill
    \begin{subfigure}[t]{0.3\linewidth}
        \includegraphics[width=\linewidth]{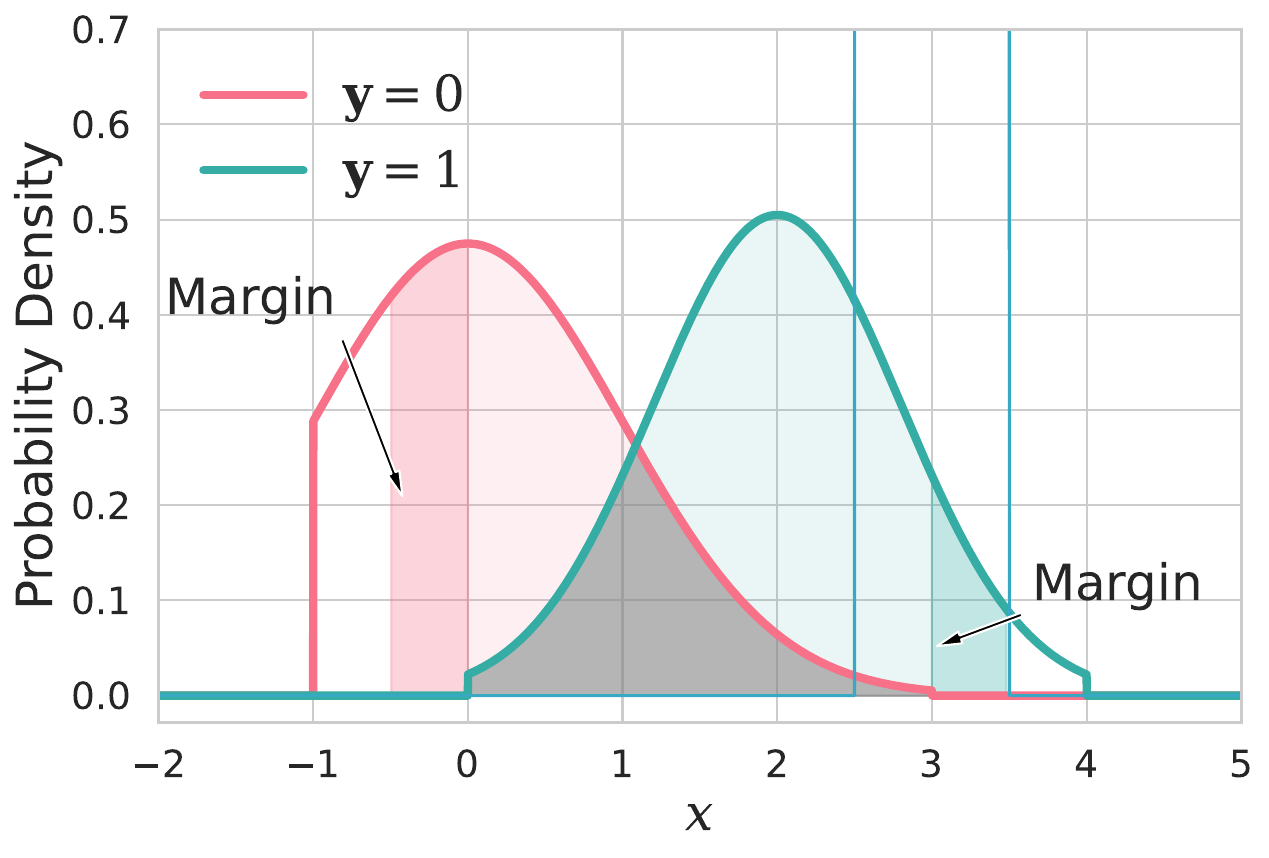}
        \caption{Demonstration of the margin contributing to $\zeta^\sharp_D$. In a 1D problem, the width of $\mathbb{K}^*_D$ is essentially the vicinity width.}
        \label{fig:peripheral}
    \end{subfigure}
    \hfill
    \begin{subfigure}[t]{0.3\linewidth}
        \includegraphics[width=\linewidth]{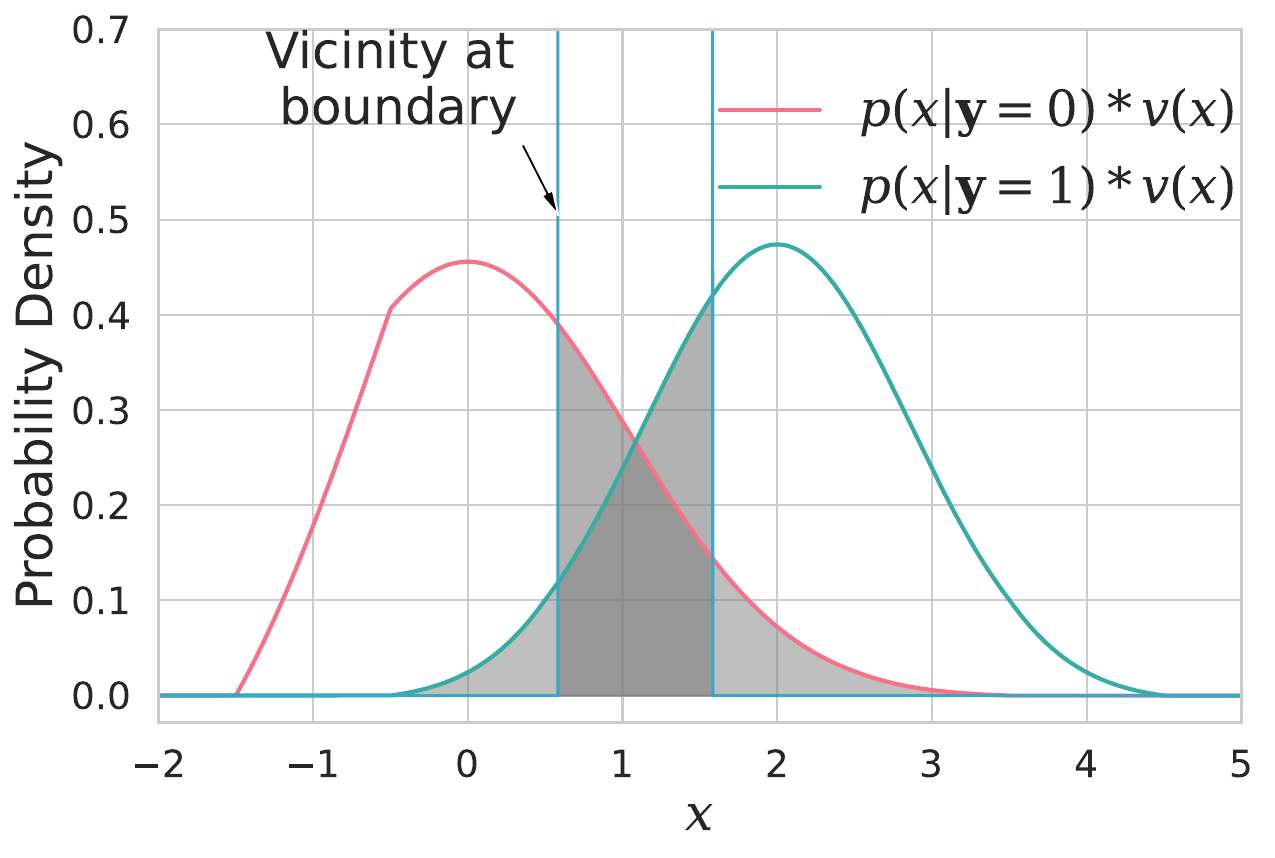}
        \caption{Union of shaded areas indicating the lower bound of irreducible robustness error and upper bound of certified robust accuracy.}
        \label{fig:ub_compute}
    \end{subfigure}
    \caption{Visualizing the convolution of distributions, the marginal contribution to Bayes error, and the bounds of robustness error and certified robust accuracy.}
    \label{fig:three_subfigures}
\end{figure}

    Suppose there is a distribution $D$ over the space $\mathbb{X}\times\mathbb{Y}$. From a local (example) perspective, an example $(\bm{x}, y)\in \mathbb{X}\times\mathbb{Y}$ assigns its label to a specific domain $\mathbb{S}\subset\mathbb{X}$ ($\mathbb{S}$ can be a vicinity) by directly altering the joint probability in $\mathbb{S}$. Specifically, this alteration is a process that adds $\Delta p(\textnormal{y}=y,\mathbf{x}=\bm{x})$ to the original $p(\textnormal{y}=y,\mathbf{x}=\bm{x})$, and adds $\Delta p(\textnormal{y}=y,\mathbf{x}=\bm{x'})$ to the original $p(\textnormal{y}=y,\mathbf{x}=\bm{x'})$ for any example $\bm{x'}\in\mathbb{S}$, where $\Delta p$ denotes a change in the joint distribution function (of $\mathbf{x}, \textnormal{y}$) such that

    \begin{equation}
    \label{eq:assign}
    \begin{aligned}
        \Delta p(\textnormal{y}=k,\mathbf{x}=\bm{x}) &= \begin{cases}
            0, &\text{if}~~ k \neq y\\
            p_\mathrm{ori}(\textnormal{y}=k,\mathbf{x}=\bm{x})\left(\frac{1}{\int_\mathbb{S}d\bm{x'}} - 1\right), &\text{otherwise}
        \end{cases} \\
        \Delta p(\textnormal{y}=k,\mathbf{x}=\bm{x'}) &= \begin{cases}
            0, &\text{if}~~ k \neq y\\
           \frac{1}{\int_\mathbb{S}d\bm{x'}} p_\mathrm{ori}(\textnormal{y}=k,\mathbf{x}=\bm{x}), &\text{otherwise}
        \end{cases}
    \end{aligned}
    \end{equation}
We then explain why this label assignment aligns with the robustness criteria. In the context of robustness, for every input in the training set, every neighbour point in the vicinity around the input gets the label of this input. Meanwhile, an input may fall within more than one vicinity. Thus, an input gets labels assigned from multiple neighbours, and each label's influence depends on its source's joint probability. Intuitively, examples with higher joint probability have a stronger influence on its vicinity.

\cref{eq:assign} captures the effect of an input's label on its neighbours, from an individual input perspective. We next set out from a distributional perspective which is supposed to match our individual input perspective. When all examples in the original distribution concurrently assign labels to their respective vicinity, the effect is equivalent to convolving this given distribution with the vicinity (function). This convolved distribution represents the target of certified robustness optimization, as captured by \cref{thm:conv}.
\begin{theorem}
\label{thm:conv}
    Given a distribution $D$ for classification, optimising for higher certified robustness does not optimise the classifiers to fit $D$. Rather, it optimises classifiers towards $D*v$, \emph{i.e.}, convolved distribution between $D$ and vicinity $v(\bm{x})$.
\end{theorem}

\begin{proof}
    The objective of robustness optimises classifiers to satisfy Formula~(\ref{eq:blind}). Therefore, the objective is to maximise
    \begin{equation}
    \begin{aligned}
        \operatorname{E}_{(\mathbf{x},\mathbf{y})\sim D} &\left[\lfloor\int_\mathbb{X} v (\mathbf{x} - \bm{x'})\cdot\bm{1}_{ \mathbf{y} = h(\bm{x'}) } d \bm{x'}\rfloor\right]\\
        = \sum_k\int &\lfloor\int_\mathbb{X} v (\bm{x} - \bm{x'})\cdot\bm{1}_{ k = h(\bm{x'}) } d \bm{x'}\rfloor p(\bm{x},k) d \bm{x}\\
        = \int \sum_k&\lfloor\int_\mathbb{X} v (\bm{x} - \bm{x'})\cdot\bm{1}_{ k = h(\bm{x'}) } d \bm{x'}\rfloor p(\bm{x},k) d \bm{x}
    \end{aligned}
    \end{equation}
Here, we can see the convolution form, and according to the commutativity and associativity of convolution. We can thus obtain that the target distribution of certified training is indeed the convolved distribution of the given one.
\qed
\end{proof}

Note that \cref{thm:conv} is not particularly for existing certified training approaches but rather any approach to achieving certified robustness. Hereafter, we use $D$ to denote the original distribution, $p$ to denote the conditional distributions (of each class) from $D$, $D'$ to denote the convolved distribution, and $q$ to denote the conditional distributions (of each class) from $D'$. Thus, the Bayes error of $D'$ can be expressed as
\begin{equation}
    \beta_{D'} = \operatorname{E}_{(\mathbf{x}, \textnormal{y}) \sim D'}\left[1 - \max_k q(\textnormal{y}=k|\mathbf{x})\right]
\end{equation}
The subsequent question is to study the Bayes error with respect to the convolved distribution $D'$. \cref{thm:bayesgrow} suggests that Bayes error grows when $D$ is transformed into $D'$, as illustrated in \cref{fig:runex_conv}.

\begin{theorem}
\label{thm:bayesgrow}
    Given a distribution $D$ for classification, its convolved distribution $D'$ has an equal or larger Bayes error, \emph{i.e.}, $\beta_{D} \le \beta_{D'}$.
\end{theorem}
\begin{proof}
    Consider $D$ is a distribution of random variables $\mathbf{x}$ and $\textnormal{y}$. Let $p_k(\bm{x})$ be the conditional distribution of $\mathbf{x}$ given $\textnormal{y}=k$. We need to prove that the Bayes error between $p_k$ is less than or equal to the Bayes error between $p_k*v$, where $v$ is a probability density function (PDF). First, let us prove that $((\max_k(p_k))*v)(\bm{x}) \ge \max_k ((p_k * v)(\bm{x}))$. Expanding both sides, we get $\int \max_k(p_k(\bm{x} - \bm{x'})v(\bm{x'}))d\bm{x'} $ at left. We get $\max_k(\int p_k(\bm{x} - \bm{x'})v(\bm{x'})d\bm{x'})$ at right. We can see that left $\ge\int p_k(\bm{x} - \bm{x'})v(\bm{x'})d\bm{x'} $ for any $k$. Therefore, the maximum can be brought out from the integral and thus the left side is proved to be greater than or equal to the right side. Then, we use the equality that the integral of $((\max_k p_k)*v)(\bm{x})$ is actually the same as the integral of $(\max_k p_k)(\bm{x})$. This is because $v$ itself is a PDF. Therefore, we get 
    \begin{equation}
        \int \left(1 - (\max_k p_k)(\bm{x})\right)d\bm{x} \le \int \left(1  - \max_k ((p_k * v)(\bm{x}))\right)d\bm{x}
    \end{equation}\qed
\end{proof}

Intuitively, when robustness is required, new labels are assigned to data in the vicinity of the training inputs. But, these new labels sometimes contradict the original labels or contradict themselves. As a result, the convolved distribution invariably exhibits larger uncertainty, represented by an increased Bayes error. For instance, let us consider a separable distribution with a unique boundary. The condition $d(\bm{x}, \bm{x'})\le\epsilon$ implies that $\bm{x}$ and $\bm{x'}$, near the boundary, should be assigned the same label even if their ground-truth labels are different, leading to a non-zero Bayes error.

\subsection{Irreducible robustness error and robustness upper bound}
\label{sec:bound}

We have proved that the optimal robustness is equal to or lower than the optimal accuracy. We now would like to find a quantitative upper bound for robustness. We first define the irreducible expected error rate across all classifiers regarding robustness, as expressed in \cref{eq:robustinf} where $\zeta^\sharp_D$ represents the irreducible robustness error on distribution $D$.
\begin{equation}
\label{eq:robustinf}
\begin{aligned}
    \zeta^\sharp_D = &\inf_{\text{measurable}~h} \operatorname{E}_{(\mathbf{x}, \textnormal{y}) \sim D}\left[1 -\bm{1}_{\operatorname{Rob}\big(h, \mathbf{x}, \textnormal{y}; \mathbb{V}_{\mathbf{x}}\big)}\right] 
\end{aligned}
\end{equation}
This concept is analogous to \cref{eq:bayeserror2}, where the Bayes error is described as the irreducible vanilla error rate achievable by any classifier. Then, the upper bound of expected robustness is 1 minus the lower bound of $\zeta^\sharp_D$.

Recall \cref{def:robustness}, the condition $\operatorname{Rob} (h, \bm{x}, y; (d, \epsilon) )$ holds only if Formula~\ref{eq:blind} is met, where $(\bm{x}, y)\in\mathbb{X}\times\mathbb{Y}$. Nevertheless, Formula~\ref{eq:blind} alone is not a sufficient condition for $\operatorname{Rob}\big(h, \bm{x}, y; (d, \epsilon)\big)$. According to \cref{def:ae}, $\operatorname{Rob}\big(h, \bm{x}, y; (d, \epsilon)\big)$ also requires that no input in the vicinity of $\bm{x}$ should be classified incorrectly, as expressed in Formula~\ref{eq:correct}. Formally, $\operatorname{Rob}\big(h, \bm{x}, y; (d, \epsilon)\big) \Longleftrightarrow \mathrm{Formula}~\ref{eq:blind}\land \mathrm{Formula}~\ref{eq:correct}$.
\begin{equation}
\label{eq:correct}
    \lnot \exists~ (\bm{x'}, y') \in \mathbb{X}\times\mathbb{Y}.~ p(\bm{x'}, y') > 0 ~\land~ d(\bm{x}, \bm{x'}) \le \epsilon ~\land~ h(\bm{x'})\neq y'
\end{equation}
\cref{eq:correct} suggests that if the ground-truth labels of inputs in the vicinity of $\bm{x}$ are different from the labels of $\bm{x}$, then a prerequisite of robustness is missing such that robustness cannot be attained. The conjunction of Formula~\ref{eq:blind} and \ref{eq:correct} clarifies that robustness asks for general correctness across the (local) input domain, rather than just local consistency. From this conjunction, we can derive that for a classifier to attain robustness at an input, it is necessary that the posterior probability associated with this input is entirely certain, which is formally captured in \cref{thm:certainrobust}. Further, given a distribution, the proportion of examples with uncertain labels can serve as a lower bound for the proportion of examples without robustness.
\begin{theorem}
\label{thm:certainrobust}
    Given a distribution $D$ over $\mathbb{X}\times\mathbb{Y}$, the irreducible robustness error is greater than or equal to the probability that an input is in $\mathbb{K}_D$.
    \begin{equation}
    \label[ineq]{ineq:certainrobust}
        \zeta^\sharp_D \ge \int \lceil1 -\max_k p(\textnormal{y}=k|\mathbf{x}=\bm{x})\rceil ~p(\bm{x}) d \bm{x} \ge \beta_D
    \end{equation}    
    When $ \lceil1-\max_k p(\textnormal{y}=k|\mathbf{x}=\bm{x})\rceil = 0 $, there is one and only one class has a posterior probability of 1 at input $\bm{x}$, resulting in a non-zero contribution to the Bayes error.
\end{theorem}
\begin{proof}
    Assume some classifier $h$ attains robustness at input $\bm{x}$, and the posterior probability is not certain, \emph{i.e.}, $1-\max_k p(\textnormal{y}=k\mid \mathbf{x}=\bm{x}) >0$. The latter infers that there exists some (non-zero probability) examples of $(\bm{x}, y_1)$ pair and some $(\bm{x}, y_2)$ pair, and $y_1\neq y_2$. The prediction for $\bm{x}$ differs from at least one of either  $y_1$ and $y_2$. Formally, the latter condition in our assumption entails $ (h(\bm{x})\neq y_1 \lor h(\bm{x})\neq y_2)$, which then entails 
    \begin{equation}
    \label{eq:disj}
    \begin{aligned}
        &\exists \bm{x'}\in\mathbb{V}_{\bm{x}}.~ h(\bm{x'})\neq y_1 \lor h(\bm{x'})\neq y_2  &\text{(because } \bm{x}\in\mathbb{V}_{\bm{x}} \text{)}\\
        &\exists \bm{x'}\in\mathbb{V}_{\bm{x}}\, \exists y'\in \mathbb{Y}.~p(\bm{x'}, y') > 0 ~\land~ h(\bm{x'})\neq y' & \text{(Disjunction Elimination)}
    \end{aligned}
    \end{equation}

    Condition (\ref{eq:disj}) contradicts the former condition in our assumption, \emph{i.e.}, Condition~\ref{eq:correct}. Thus, robustness may only be attained if there is no label uncertainty at an input.
\qed
\end{proof}

Uncertainty contributes to an irreducible error in both vanilla accuracy and robustness. The irreducible robustness error is at least the Bayes error. We are further interested in refining this boundary in scenarios where we know the value of the Bayes error but lack information about the posterior probabilities. To this end, we develop \cref{cor:certainrobust}.
\begin{corollary}
\label{cor:certainrobust}
Given a distribution $D$ over $\mathbb{X}\times\mathbb{Y}$, its irreducible robustness error is at least as large as the Bayes error multiplied by the number of classes divided by one less than the number of classes, \emph{i.e.},
    \begin{equation}
        \zeta^\sharp_D  \ge \frac{\abs{\mathbb{Y}}}{\abs{\mathbb{Y}}-1}\beta_D
    \end{equation}
where $\abs{\mathbb{Y}}$ denotes the number of classes.
\end{corollary}
\begin{proof}
    we have that $\zeta^\sharp_D\ge \int_{\mathbb{K}_D}p(\bm{x})d\bm{x} =  \int_{\mathbb{K}_D}(1 - \max_k p(\textnormal{y}=k|\mathbf{x}=\bm{x}) + \max_k p(\textnormal{y}=k|\mathbf{x}=\bm{x}))p(\bm{x})d\bm{x} = \beta_D + \int_{\mathbb{K}_D}(\max_k p(\textnormal{y}=k|\mathbf{x}=\bm{x}))p(\bm{x})d\bm{x} \ge \int_{\mathbb{K}_D} p(\bm{x})/\abs{\mathbb{Y}}d\bm{x} + \beta_D$. Thus, we can prove that $\int_{\mathbb{K}_D}p(\bm{x})d\bm{x}\ge\beta_D\abs{\mathbb{Y}}/(\abs{\mathbb{Y}}-1)$
\qed\end{proof}
In \cref{thm:certainrobust} and \cref{cor:certainrobust}, the lower bounds for the $\zeta^\sharp_D$ are established based on that a single input needs to have a deterministic label. Still, there are additional conditions that, if unmet, will prevent a classifier from attaining robustness for a given input. For instance, we can expand the certainty requirement from a single input to encompass any input within its vicinity. The input neighbours in the vicinity with uncertain labels can also contribute to the irreducible robustness error. Given an input $\bm{x}$ such that $ \bm{x}\notin \mathbb{K}_D$, if there exists an $\bm{x'}$ within this vicinity of $\bm{x}$ such that $\bm{x'}\in \mathbb{K}_D$, robustness at $\bm{x}$ cannot be attained. All such $\bm{x}$ forms a domain $\mathbb{K^*}_D$. $\mathbb{K^*}_D$ can be considered as a thin margin around $\mathbb{K}_D$, as shown in \cref{fig:peripheral}.  This expansion results in a more stringent condition. Consequently, we will likely identify a tightened lower bound for $\zeta^\sharp_D$.
\begin{corollary}
\label{cor:irrewithmargin}
    Given a distribution $D$ over $\mathbb{X}\times\mathbb{Y}$, then
    \begin{equation}
    \label{eq:irrewithmargin}
        \zeta^\sharp_D \ge 2\cdot\epsilon_\mathrm{eff}\cdot p_{\min} \cdot\left(\int_{\mathbb{K}_D} d \bm{x}\right)^{\frac{\dim\mathbb{X}-1}{\dim\mathbb{X}}} + \int_{\mathbb{K}_D}  p(\bm{x}) d \bm{x}
    \end{equation}
    where $\epsilon_\mathrm{eff}$ denotes the radius of the vicinity according to the definition of robustness, \emph{e.g.}, for $L^2$-perturbation, $\epsilon_\mathrm{eff}$ equals to the radius $\epsilon$. For general perturbations,
    \begin{equation}
        \frac{\pi^{\dim\mathbb{X}/2}}{\Gamma\left(\frac{\dim\mathbb{X}}{2} + 1\right)} \epsilon_\mathrm{eff}^{\dim\mathbb{X}} = \int_\mathbb{X} \lceil v(\bm{x})\rceil d\bm{x}
    \end{equation}    
\end{corollary}

\begin{proof}
   
    $\mathbb{K}^*_D$ emerges when a perturbation vicinity sweeps along the boundary of $\mathbb{K}_D$ and the isoperimetric inequality suggests that the volume of this marginal domain is minimized when both vicinity and $\mathbb{K}_D$ are $\dim\mathbb{X}$-spheres.    
    
    Thus, a lower bound of the volume of $\mathbb{K}^*_D$ can be expressed as the volume difference between two concentric spheres which is again greater than the product of their radius difference and the surface area of the inner sphere. Thus, the volume of $\mathbb{K}^*_D$ is lower bounded by
    \begin{equation}
    \label{eq:marginlb}
        \epsilon_\mathrm{eff} \cdot \operatorname{vol}(\mathbb{K}_D)^{\frac{\dim\mathbb{X}-1}{\dim\mathbb{X}}} \cdot\frac{2 \pi^{\dim\mathbb{X}/2}}{\Gamma\left(\frac{\dim\mathbb{X}}{2}\right)} \left( \pi^{-\dim\mathbb{X}/2} \Gamma\left(\frac{\dim\mathbb{X}}{2} + 1\right) \right)^{\frac{\dim\mathbb{X} - 1}{\dim\mathbb{X}}} 
    \end{equation}
    where $\Gamma$ represents the gamma function, and for all positive real numbers $\Gamma (z)=\int _{0}^{\infty }t^{z-1}e^{-t}\,dt$. We further simplify \cref{eq:marginlb} to $\epsilon_\mathrm{eff} \cdot \operatorname{vol}(\mathbb{K}_D)^{(\dim\mathbb{X}-1)/(\dim\mathbb{X})} \cdot \gamma$, where $\gamma \ge 2 \mbox{ for } \dim\mathbb{X} > 1$
    and a lower bound of $\operatorname{vol}(\mathbb{K}^*_D)$ is thus $2\cdot\epsilon_\mathrm{eff} \cdot(\int_{\mathbb{K}_D} d \bm{x})^{(\dim\mathbb{X}-1)/(\dim\mathbb{X})}$. In very high dimensions, the (minimum) volume of $\mathbb{K}^*_D$ is almost linearly related to the volume of $\mathbb{K}_D$. The irreducible error contributed by the marginal domain to robustness can be expressed as $\int_{\mathbb{K}^*_D} p(\bm{x}) d \bm{x}$. It is greater than or equal to $\int_{\mathbb{K}^*_D} p_{\min} d \bm{x}$, where $p_{\min} = \min_{\bm{x}} p (\bm{x})$. Thus, this irreducible error $\ge p_{\min} \cdot\operatorname{vol}(\mathbb{K}^*_D)$, contributes to the irreducible robustness error as the first term in \cref{eq:irrewithmargin}. This corollary is particularly useful if we know the non-zero $p_{\min}$ of the distribution.
\hfill\qed\end{proof}

In short,  \cref{thm:certainrobust}, \cref{cor:certainrobust}, and \cref{cor:irrewithmargin} suggest how we can get lower bounds of irreducible robustness error $\zeta^\sharp_D$ from the original distribution $D$, with lower bound from \cref{cor:irrewithmargin} being the tightest among three. Given a distribution, $\zeta^\sharp_D$ has two sources. One is the examples that have uncertain labels (which contribute to the error directly), and the other is the examples that have neighbours whose labels are uncertain (which contribute to the error indirectly). Additionally, when the Bayes error $\beta_D$ of a distribution is non-zero, the irreducible error of robustness $\zeta^\sharp_D$  is also non-zero and is greater than the Bayes error.

Theoretically, there is another way to tighten the bound provided by \cref{thm:certainrobust}. If we know the convolved distribution $D'$ obtained in \cref{sec:conv}, the $\zeta^\sharp_D$ can be calculated as  $p(\mathbf{x}\in\mathbb{K}_\mathrm{D'})$, \emph{i.e.}, the probability (in convolved distribution) that input has a deterministic label. Thus,
    \begin{equation}
        \zeta^\sharp_D =\operatorname{E}_{(\mathbf{x}, \textnormal{y}) \sim D'}\left[1 - \lfloor\max_k q(\textnormal{y}=k|\mathbf{x}) \rfloor\right] 
    \end{equation}
Since $D' = D*v$, as vicinity size grows, the Bayes error of $D'$ also grows, and thus the irreducible robustness error $\zeta^\sharp_D$ grows.

The least upper bound of robustness on a given data distribution $D$ can then be written as $1-\zeta^\sharp_D$, and 1 minus any lower bound of $\zeta^\sharp_D$ presented above serves as an upper bound of robustness on a given data distribution $D$. These upper bounds are directly derived from the data distribution $D$ and the vicinity function $v$, independent of any specific classifier.

Although we have been using both Formulae \ref{eq:blind} and (\ref{eq:correct}) throughout this subsection, the existing studies only rely on Formula~\ref{eq:blind}~\cite{li2023sok} for practical evaluation of certified robust accuracy. Intuitively, we sometimes do not know the true label of a neighbour $\bm{x'}$ in an input $\bm{x}$'s vicinity, and thus use the $\bm{x}$'s label instead. Consequently, the correctness of $\bm{x'}$ prediction is neglected. Instead, only the consistency between predictions on $\bm{x'}$ and $\bm{x}$, as well as the correctness of prediction on $\bm{x}$, are considered. This simplification could result in a different certified robust accuracy for classifiers and exceed our upper bounds of robustness on a given data distribution (\cref{thm:certainrobust}). To this end, we also present the irreducible robustness error $\zeta_D$ in \cref{eq:ub_compute} and the corresponding upper bound for such robustness on a given data distribution $1-\zeta_D$. We use \cref{fig:ub_compute} to illustrate its effect.
\begin{equation}
\label{eq:ub_compute}
    \zeta_D =\int_{\mathbb{K}_{D^\dagger}} q(\bm{x})d\bm{x} + \int_{\mathbb{X}\setminus\mathbb{K}_{D^\dagger}} \left(1 - \max_k q(\textnormal{y}=k|\mathbf{x}=\bm{x})\right) q(\bm{x}) d \bm{x}
\end{equation}
where $D^\dagger$ is a distribution obtained from convolving the vicinity function $v$ and the ``hardened'' distribution of $D'$, \emph{i.e.}, each $p_\mathrm{hard}(\textnormal{y}=k_{\max}|\mathbf{x}=\bm{x}) = q(\bm{x})$ and for other $k\neq k_{\max}$, $p_\mathrm{hard}(\textnormal{y}=k|\mathbf{x}=\bm{x}) = 0$. Then, $p_{D^\dagger} = p_\mathrm{hard}*v$. Recall that $q$ is the conditional distribution of $D'$. In \cref{eq:ub_compute}, its first term suggests no input close to the boundary can attain robustness. For inputs not close to the boundary, as indicated by the second term, their optimal robustness on a given data distribution depends on the correctness of the prediction. In terms of \cref{fig:ub_compute}, the first term corresponds to the shaded area bounded by the vicinity, and the second term corresponds to all shaded areas outside the curve. Although \cref{eq:ub_compute} has tackled the label-missing challenges in practice, this theoretical evaluation of the irreducible error (in certified robust accuracy) could still rely on the knowledge of distribution. Thus, distribution estimating techniques are also needed when facing sampled data from an unknown distribution.

\section{Experiment and Results}
\label{sec:experiment}
In this section, we empirically test our results discussed above by designing and answering three research questions: 1) does certified training always result in a classifier on a distribution with a higher Bayes error; 2) is our computed upper bound of robustness indeed higher than the robustness achieved by all the existing certified training classifiers; and 3) does the upper bound of robustness change when the vicinity increases, and if so how does it change? 

The experiments are conducted with four data sets: two synthetic ones (\emph{i.e.}, Moons and Chan~\cite{chen2023evaluating}) and two standard benchmarks (\emph{i.e.}, FashionMNIST~\cite{xiao2017fashion} and CIFAR-10~\cite{krizhevsky2009learning}). Moons is used for binary classification with two-dimensional features, where each class's distribution is described analytically with specific likelihood equations, and uses a three-layer Multi-Layer Perceptron (MLP) neural network for classification. The Chan data set, also for binary classification with two-dimensional features, differs in that it does not follow a standard PDF pattern, requiring kernel density estimation (KDE) for non-parametric PDF estimation, and also uses the three-layer MLP. FashionMNIST, a collection of fashion item images, involves a 10-class classification task with 784-dimensional inputs (28$\times$28 pixel grayscale images). Each class has an equal prior probability, and their conditional distributions are estimated non-parametrically using KDE. CIFAR-10 uses images with a resolution of 32$\times$32 pixels. Similar to FashionMNIST, it has a balanced class distribution and is estimated using KDE. We use a seven-layer convolutional neural network (CNN-7)~\cite{shi2021fast} as the classifier of both FashionMNIST and CIFAR-10. We adopt a direct approach~\cite{ishida2022performance} to compute the original Bayes error of both FashionMNIST and CIFAR-10~\cite{ishida2022performance}.

To train the classifiers, two approaches are adopted, \emph{i.e.}, empirical error minimization (ERM~\cite{vapnik1999nature}) for standard training, and the state-of-the-art (SOTA) small-box method for certified training~\cite{muller2022certified}. The performance of these classifiers is evaluated using two metrics: vanilla accuracy and certified robust accuracy~\cite{muller2022certified}. Note that, certified robust accuracy measures the proportion of predictions that can be certified as robust in terms of satisfying Formula~\ref{eq:blind}. \\

\begin{figure}[t]
  \centering
  \begin{subfigure}[t]{0.23\linewidth}
    \centering
    \captionsetup{justification=centering}
    \includegraphics[width=\linewidth]{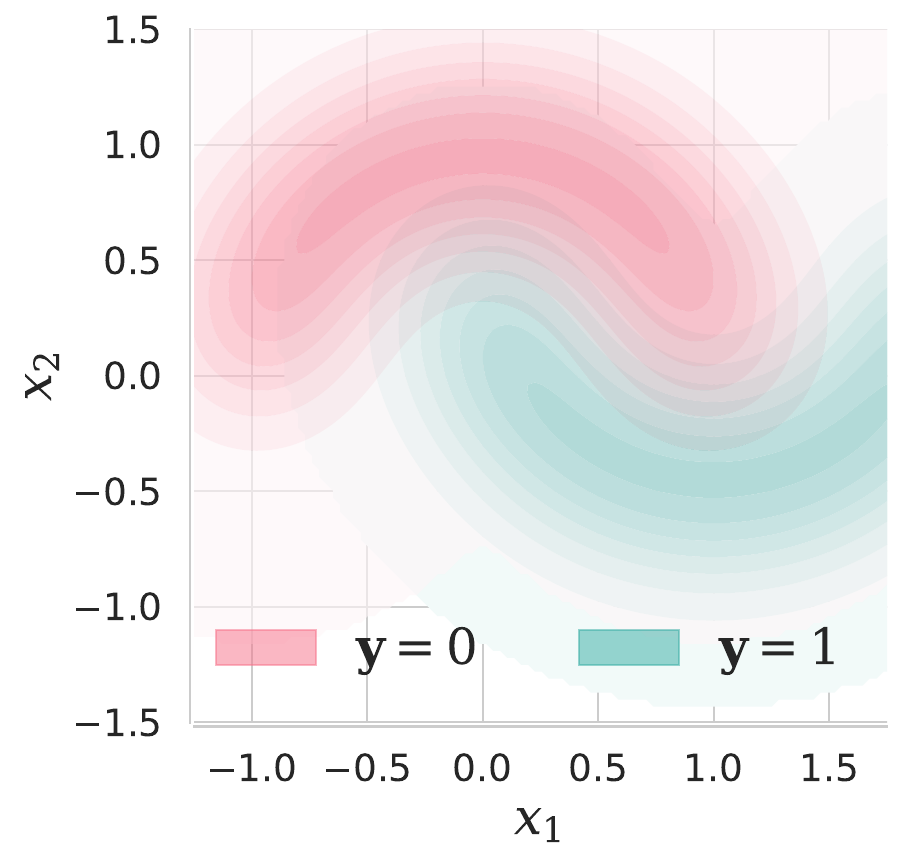}
    \caption{Original\\$\beta_{\text{Moons}}=8.54\%$}
  \end{subfigure}%
  \hfill 
  \hfill
  \begin{subfigure}[t]{0.23\linewidth}
    \centering
    \captionsetup{justification=centering}
    \includegraphics[width=\linewidth]{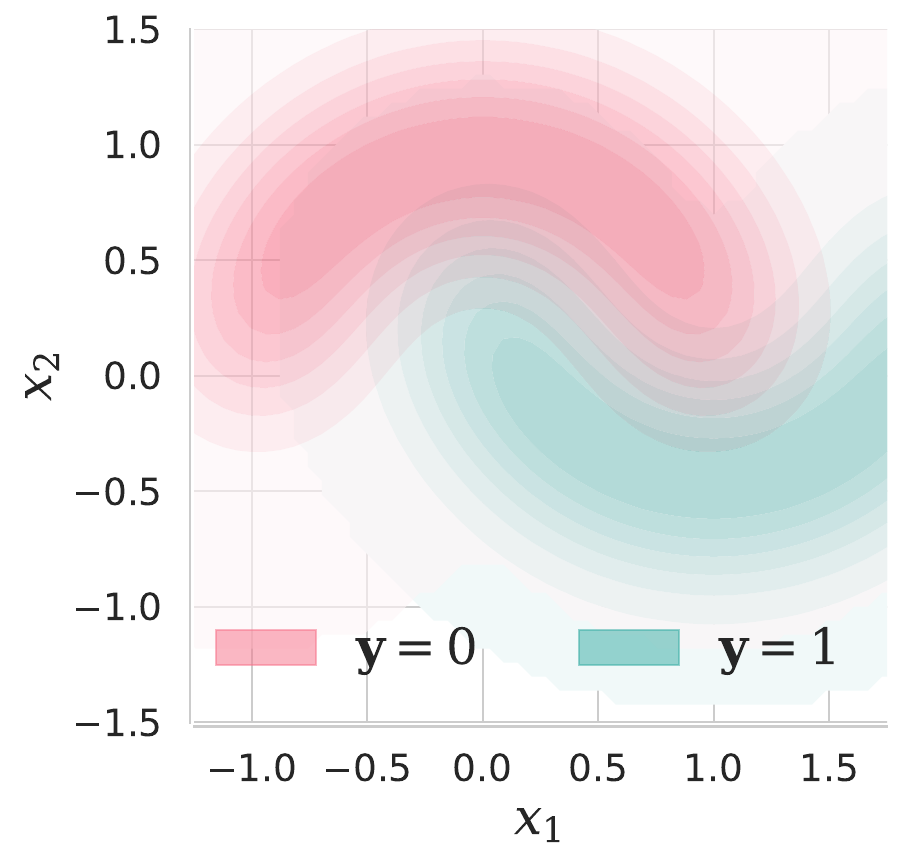}
    \caption{Convolved\\$\beta_{\text{Moons}'}=9.24\%$}
  \end{subfigure}%
  \hfill
  \begin{subfigure}[t]{0.23\linewidth}
    \centering
    \captionsetup{justification=centering}
    \includegraphics[width=\linewidth]{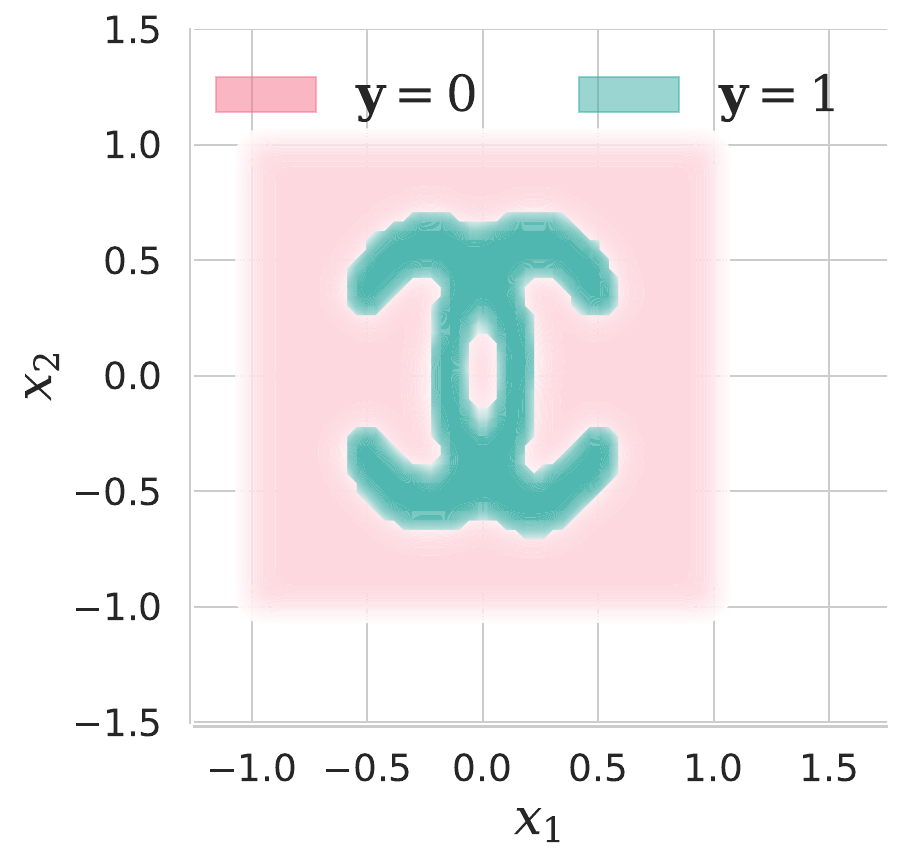}
    \caption{Original\\$\beta_{\text{Chan}}=5.39\%$}
  \end{subfigure}%
  \hfill
  \begin{subfigure}[t]{0.23\linewidth}
    \centering
    \captionsetup{justification=centering}
    \includegraphics[width=\linewidth]{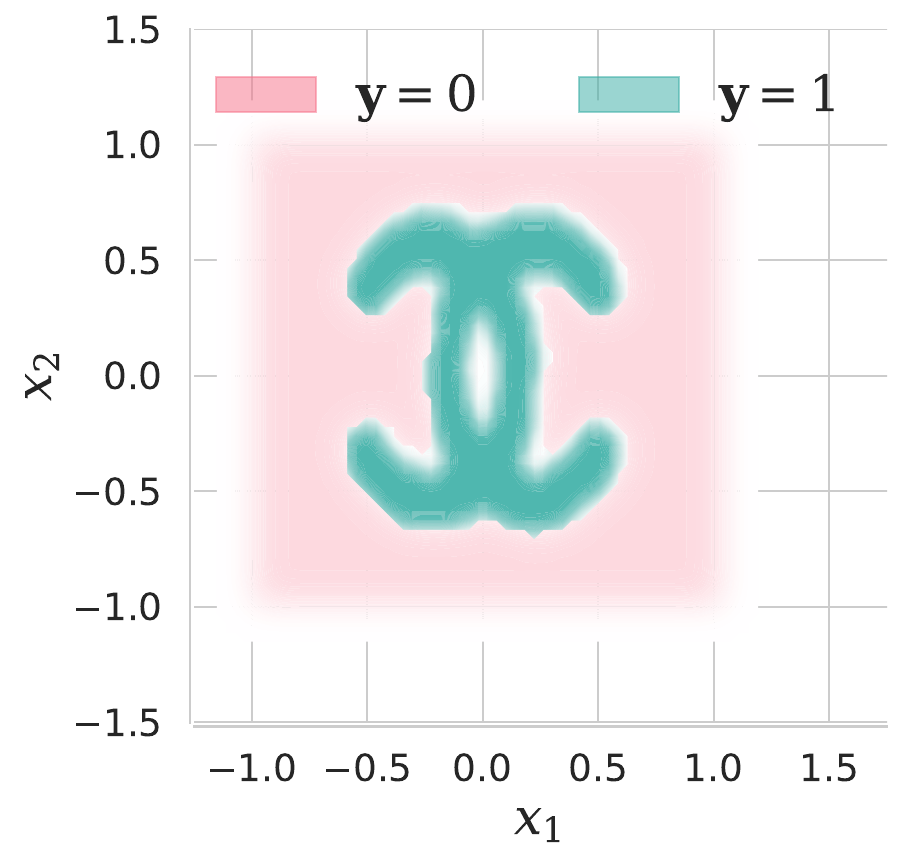}
    \caption{Convolved\\$\beta_{\text{Chan}'}=9.66\%$}
  \end{subfigure}%
  \begin{subfigure}[t]{0.05\linewidth}
    \centering
    \includegraphics[width=\linewidth, trim=2.cm -2.5cm 0.32cm 0.cm, clip]{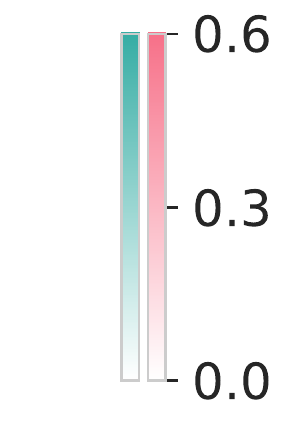}
  \end{subfigure}
  \caption{The conditional distribution before and after convolution for (a, b) Moons and (c, d) Chan. For both Moons and Chan, $L^\infty$ size is set at $\epsilon=0.15$. We also report the Bayes error to show the change of inherent uncertainty in each distribution.
  \label{fig:convolved}}
\end{figure}

\noindent \emph{RQ1: Does the Bayes error indeed grow when certified training is applied?} We would like to check if the Bayes error indeed sees a growth when certified training is used. To do that, we first need to obtain the altered distribution used in the context of certified training. As explained in \cref{sec:conv}, certified training extends the label of an input to its vicinity, and thus results in a convolutional effect across the entire given distribution. Therefore, we can obtain a convolved distribution of each given distribution with each vicinity. Then, we compare the original distribution and the convolved distribution of a given data set, and the Bayes error of the distribution before and after convolution.

We use the Moons and Chan data sets, setting a $L^\infty$ vicinity at $\epsilon=0.15$. Then, we get the convolved distribution of each data set (using FFT-based convolution, implemented through \texttt{scipy.signal.fftconvolve}) and the results are demonstrated in \cref{fig:convolved}. Observing the comparison shown in \cref{fig:convolved}(a, b) and (c, d), we can see that the original distribution has gone through a ``melting'' process, \emph{i.e.}, the peaks of each distribution becomes lower, and the spread increases. For example, in \cref{fig:convolved}b, the upper moon's centre region (around $x_1=0,x_2=0.6$) has a higher concentration of inputs from the lower moon than that in (a). This is because convolution with a rectangular function, \emph{e.g.}, vicinity function in our case, is essentially smoothing the original conditional distribution.

To quantify the increased overlap between the density function of each distribution after convolution, we compute their Bayes error. For Moons, the original Bayes error (\cref{fig:convolved}a) is 8.54\%, while the Bayes error after convolution is 9.24\%. Similarly, for Chan, Bayes error increases from 5.39\% (c) to 9.66\% (d). As expected, the Bayes errors do grow, with the growth ranging from 8\% to nearly 80\%.

We find that convolving with the same vicinity function results in very different growth in the Bayes error. This is likely due to the shape of the original density function. For instance, each moon in the Moons distribution can be approximately seen as a single-modal distribution, and the density function does not have sharp changes. In contrast, the density function of each class's conditional distribution in Chan has sharper value changes at the central region (around $x_1=0,x_2=0.5$). This may suggest that the Chan distribution exhibits a larger shape change to its original distribution after convolution than Moons. Particularly, in Chan, we observe that the class with the highest probability at the central region changes. Originally, class-0 examples have a higher density in this region. However, after convolution, we can see from \cref{fig:convolved}d that this region is filled more with class-1 examples than class-0 examples. Essentially, this change shows a significant prediction change in the Bayes classifier. This is likely because convolution has a larger influence on the distributions with features with high (2D) frequencies.

In summary, by comparing the Bayes error before and after distribution alteration, we conclude that the Bayes error does increase when certified training is used, which aligns with \cref{thm:bayesgrow}. Moreover, the distribution alteration has a larger impact on distributions with high-frequency features than on originally smooth distributions. \\

\begin{figure}[t]
  \centering
  \begin{subfigure}[t]{\linewidth}
    \centering
    \includegraphics[width=\linewidth, trim=0.5cm 0.4cm 1cm 1.cm, clip]{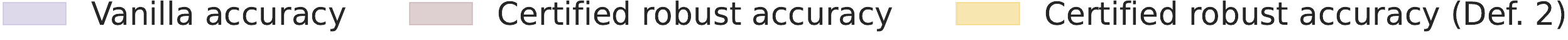}
  \end{subfigure}
  \hfill 
  \begin{subfigure}[t]{0.23\linewidth}
    \centering
    \includegraphics[width=\linewidth]{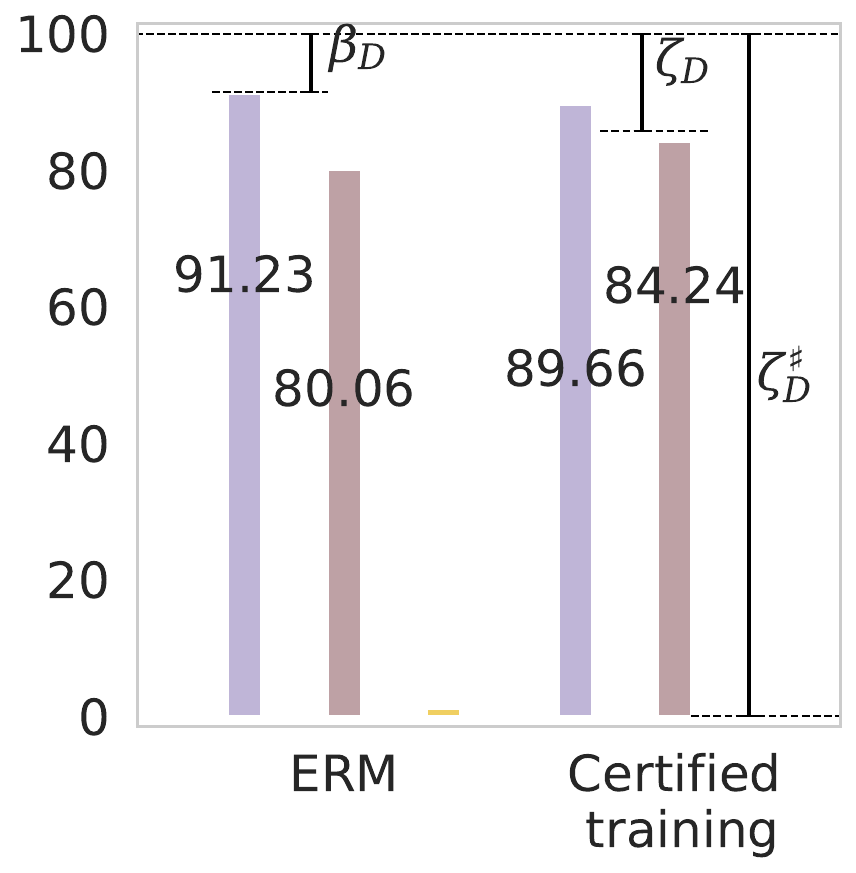}
    \caption{Moons $\beta_D = 8.54$, $\zeta_D = 14.28$, $\zeta^\sharp_D \approx 100$ (\%)}
  \end{subfigure}%
  \hfill
  \begin{subfigure}[t]{0.23\linewidth}
    \centering
    \includegraphics[width=\linewidth]{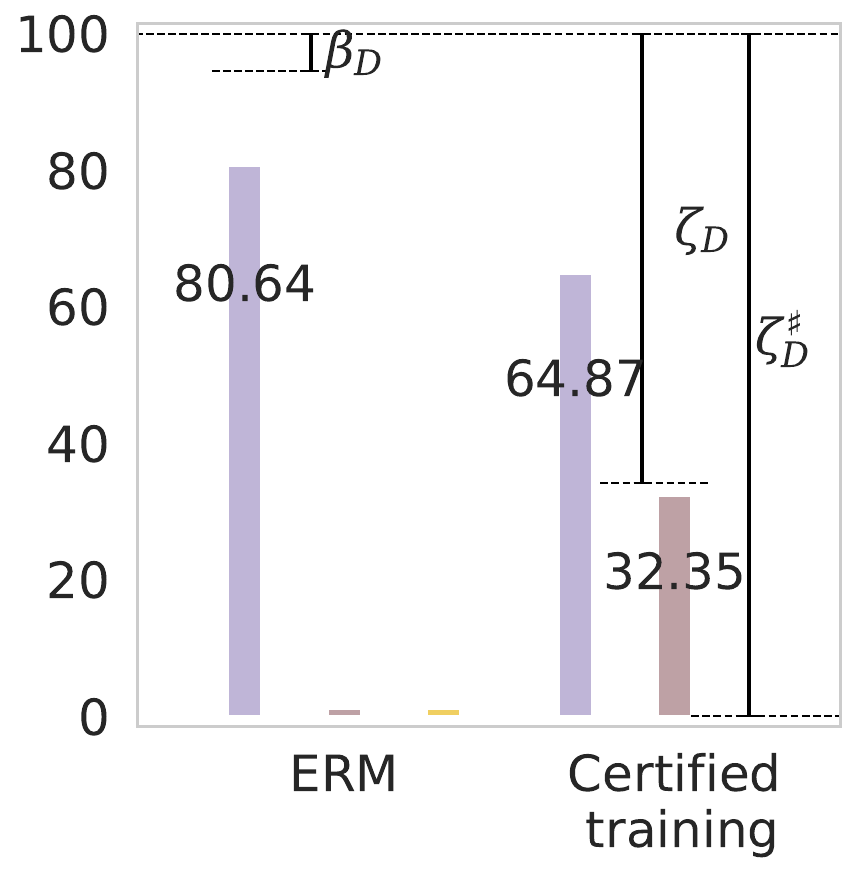}
    \caption{Chan $\beta_D = 5.38$, $\zeta_D = 65.77$, $\zeta^\sharp_D \approx 100$ (\%)}
  \end{subfigure}%
  \hfill
  \begin{subfigure}[t]{0.23\linewidth}
    \centering
    \includegraphics[width=\linewidth]{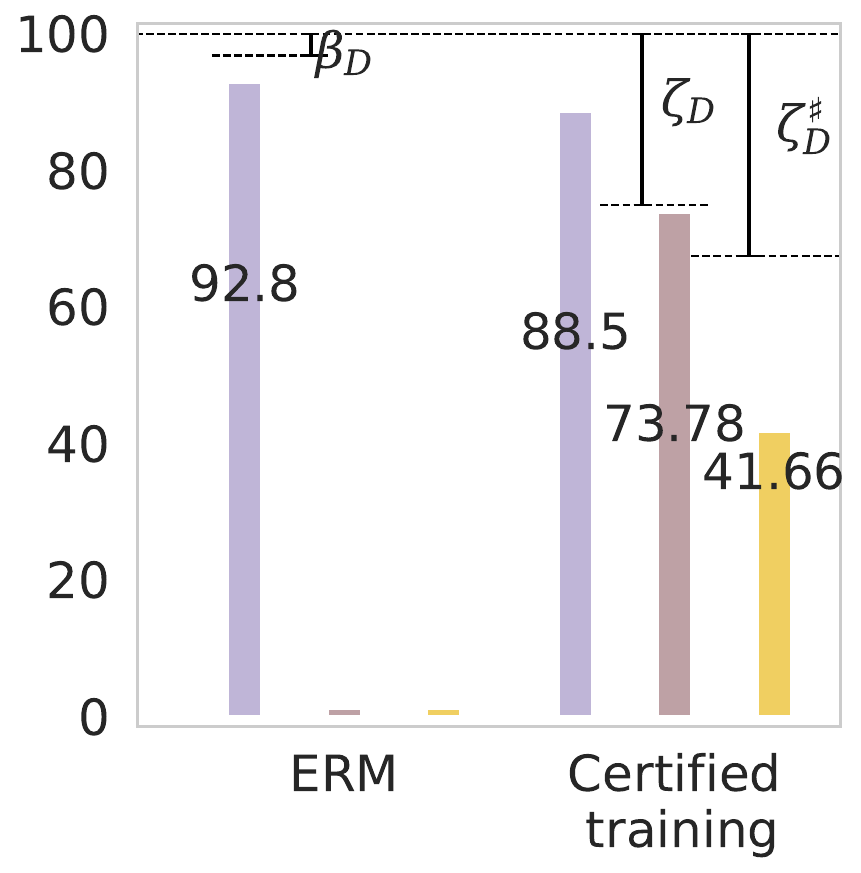}
    \caption{FashionMNIST $\beta_D = 3.15$, $\zeta_D = 25$, $\zeta^\sharp_D \approx 32.56$ (\%)}
  \end{subfigure}%
  \hfill
  \begin{subfigure}[t]{0.23\linewidth}
    \centering
    \includegraphics[width=\linewidth]{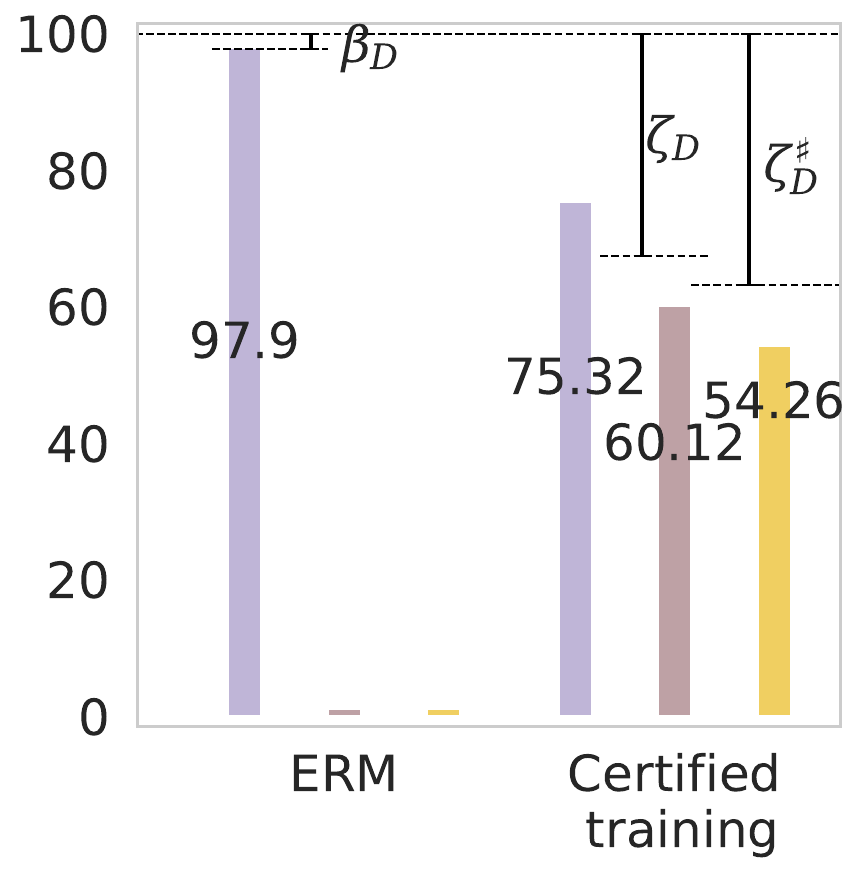}
    \caption{CIFAR-10 $\beta_D = 5.24$, $\zeta_D = 32.51$, $\zeta^\sharp_D \approx 36.81$ (\%)}
  \end{subfigure}%
  \caption{Upper bounds of robustness/accuracy and the state-of-the-art classifier's performance. The $L^\infty$ vicinity size for certified training /certified robust accuracy for each data set is $\epsilon = 0.15, 0.15, 0.1, 2/255$ for Moons, Chan, FashionMNIST, and CIFAR-10.
  \label{fig:accuracy}}
\end{figure}

\noindent \emph{RQ2: Is our upper bound of robustness empirically effective?} Next, we check whether the computed upper bound of robustness is indeed higher than the existing robustness evaluation in practice. To do that, we apply the closed-form \cref{eq:ub_compute} numerically to compute the irreducible robustness error $\zeta_D$ for each data set/distribution $D$. The upper bound of certified robust accuracy is $1 - \zeta_D$. Then, we use ERM and certified training to optimise the corresponding classifier of each data set. As such, we get two trained classifiers for each data set. For each classifier, we compute its performance metrics and compare the classifiers' performance against our upper bounds. We remark that the accuracy and certified robust accuracy may fluctuate when the sample size is not sufficiently large, as seen in \cref{fig:samplesize}. For example, if we are only given five samples, there is a high chance we get a very high or very low accuracy. For this reason, we gradually increase the sample size of test sets and observe its converged value.

\begin{figure}[t]
    \centering
    \begin{subfigure}[t]{0.4\linewidth}
    \centering
    \includegraphics[width=\linewidth]{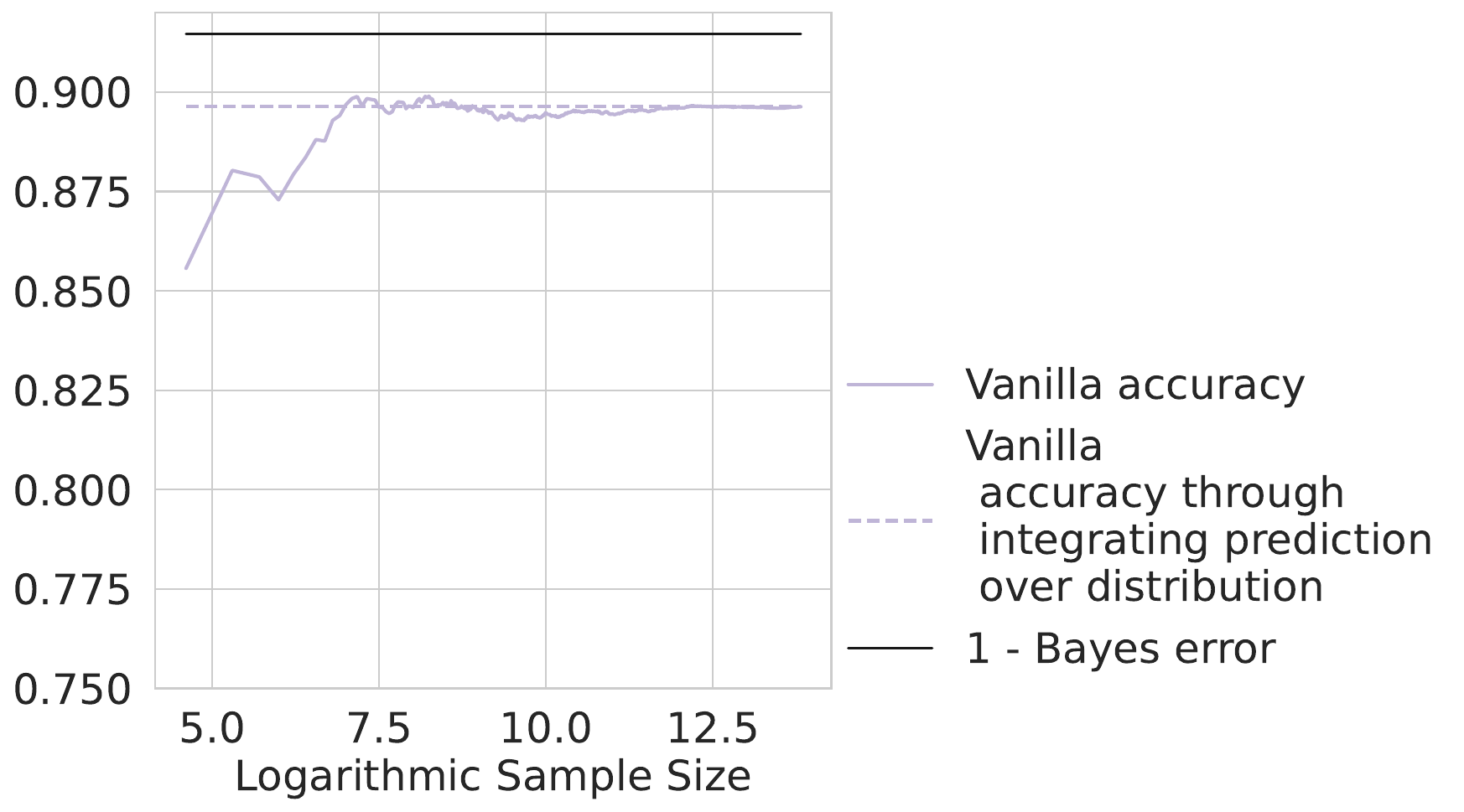}
    \end{subfigure}%
    \begin{subfigure}[t]{0.4\linewidth}
    \centering
    \includegraphics[width=\linewidth]{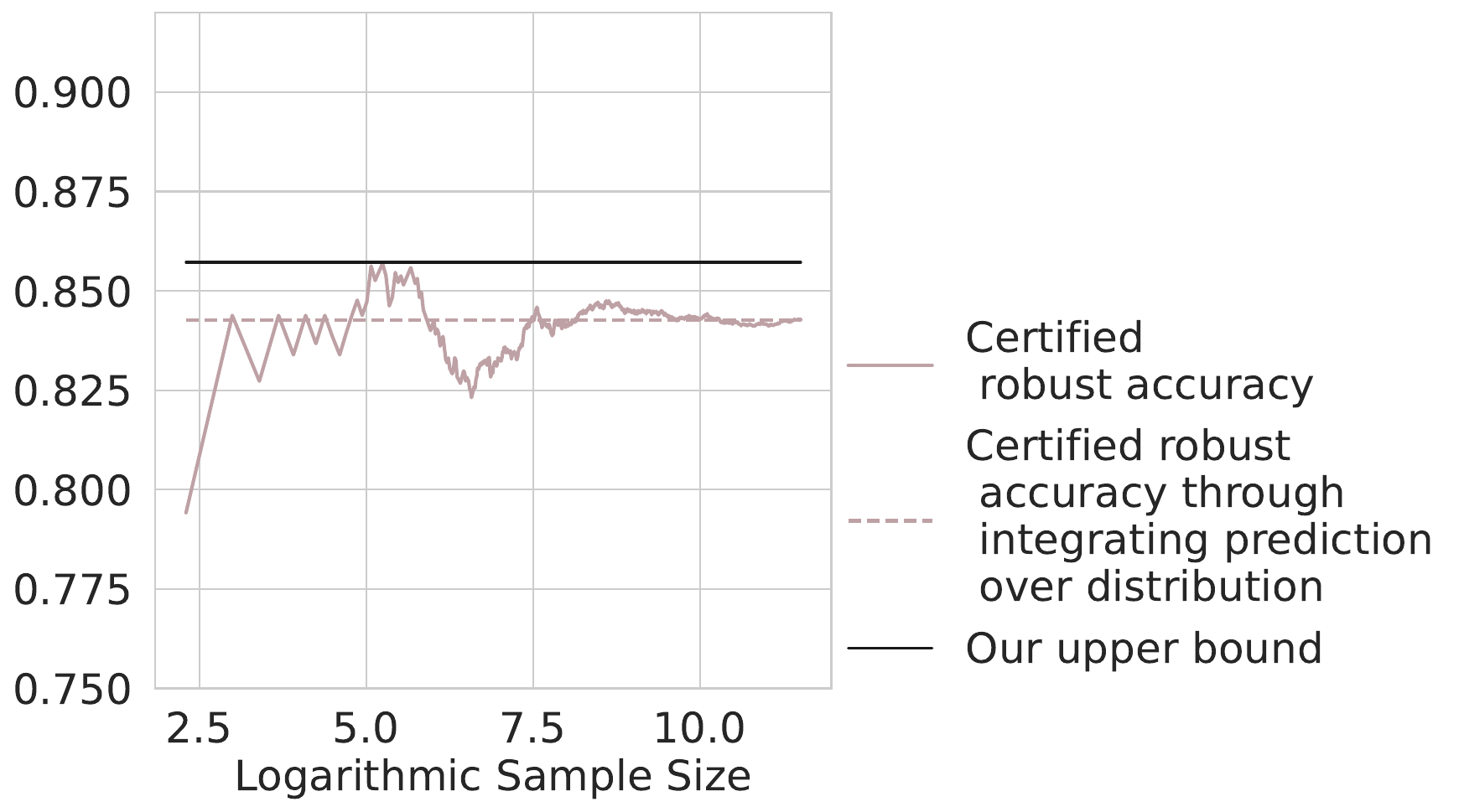}
    \end{subfigure}%
    \caption{As sample size grows, the accuracy converges to a value below \emph{1 - Bayes-error}. Similarly, the certified robust accuracy converges to a value below our upper bound of robustness. The figures are computed based on the Moons data set.}
    \label{fig:samplesize}
\end{figure}
 
The results are shown in \cref{fig:accuracy}. For each data set, the computed value $\zeta_D$ is detailed in the caption. The certified robust accuracy is represented by bars in the graph. For example, the MLP for the Moons dataset (seen in \cref{fig:accuracy}a) is trained twice. Initially, it is trained with ERM, achieving a vanilla accuracy of 91.23\%, which is nearly the optimal vanilla accuracy of 91.46\% (calculated as $1 - 8.54\%$). Here, the certified robust accuracy is about 80\% with an $L^\infty$ vicinity of $\epsilon=0.15$. When trained a second time with certified training, the MLP’s vanilla accuracy slightly decreases to 89.66\%, but its certified robust accuracy improves by 5.1\%, at 84.24\%. The improved certified robust accuracy is below the theoretical upper bound (marked by a dashed line in \cref{fig:accuracy}a, below the annotation $\zeta_D$), which is calculated to be 85.72\% ($1 - 14.28\%$). Furthermore, the gap between the certified robust accuracy of this classifier and its upper limit is relatively small, approximately 1.5\% in absolute percentage points.

Based on the result, we have multiple observations. First, we find that $1-\zeta_D$ consistently exceeds the certified robust accuracy achieved by state-of-the-art method~\cite{muller2022certified} across various datasets in \cref{fig:accuracy}. This gap, ranging from 1.5\% to 7.1\%, indicates the potential for further improving classifier robustness within these theoretical limits. For example, the Moons dataset has a small gap, suggesting limited room for improvement, while larger gaps in datasets like the Chan, FashionMNIST, and CIFAR-10 indicate more significant opportunities for increasing the robustness.

Second, we note that $\zeta_D$ consistently surpasses the Bayes error $\beta_D$ by a significant margin for all $D$. For example, in the Moons dataset, $\zeta_D$ is 14.2\%, which is 66\% higher than its $\beta_D$ of 8.54\%. This implies that robustness against perturbations is challenging, even when the inherent uncertainty of the data is considered. In datasets like FashionMNIST and CIFAR-10, despite their low Bayes error of 3.1\%-5.2\%, 
their $\zeta_D$ are at least six times higher (25.0\%-32.7\%). This indicates that factors other than inherent data uncertainty are affecting robustness. These factors are likely the newly generated uncertainty from certified training.  Moreover, 
some gaps between $\zeta_D$ and $\beta_D$  are particularly large (\emph{e.g.}, \cref{fig:accuracy}b). Such instances highlight the robustness challenges presented by each dataset can vary. Recall \cref{fig:convolved}d, the distribution of Chan can be particularly sensitive to convolution with vicinity.

Third, recall that the upper bound $\zeta_D$ and certified robust accuracy are based on  Formula~\ref{eq:blind} and \cref{eq:ub_compute}, and they do not consider the correctness of the label (\cref{def:ae}). If we take into consideration the correctness of the examples in the vicinity, we can compute a tighter bound $\zeta^\sharp_D$ based on \cref{eq:correct}, and certified robust accuracy (from Def.~\ref{def:ae}) is calculated by sampling a large finite number of neighbours of the input and evaluating their correctness. As the test sample size grows, more examples appear in the vicinity of some training samples, and the likelihood of correctly predicting all of them decreases. This result is also illustrated in the right-most columns in \cref{fig:accuracy}. As observed, certified robust accuracy (from Def.~\ref{def:ae}) is always lower than $1-\zeta^\sharp_D$. For instance, the certified robust accuracy of Moons (and Chan) decreases from 84.24\% (and 32.35\%) respectively to less than $10^{-7}$, and that of FashionMNIST (and CIFAR-10) decreases from 73.78\% (and 60.12\%) to 41.66\% (and 54.26\%) respectively. Such large reductions indicate a potential need for rethinking the robustness requirement, which may lead to different ways of defining and achieving robustness. \\

\begin{figure}[t]
    \centering
    \begin{subfigure}[t]{0.23\linewidth}
    \centering
    \includegraphics[width=\linewidth]{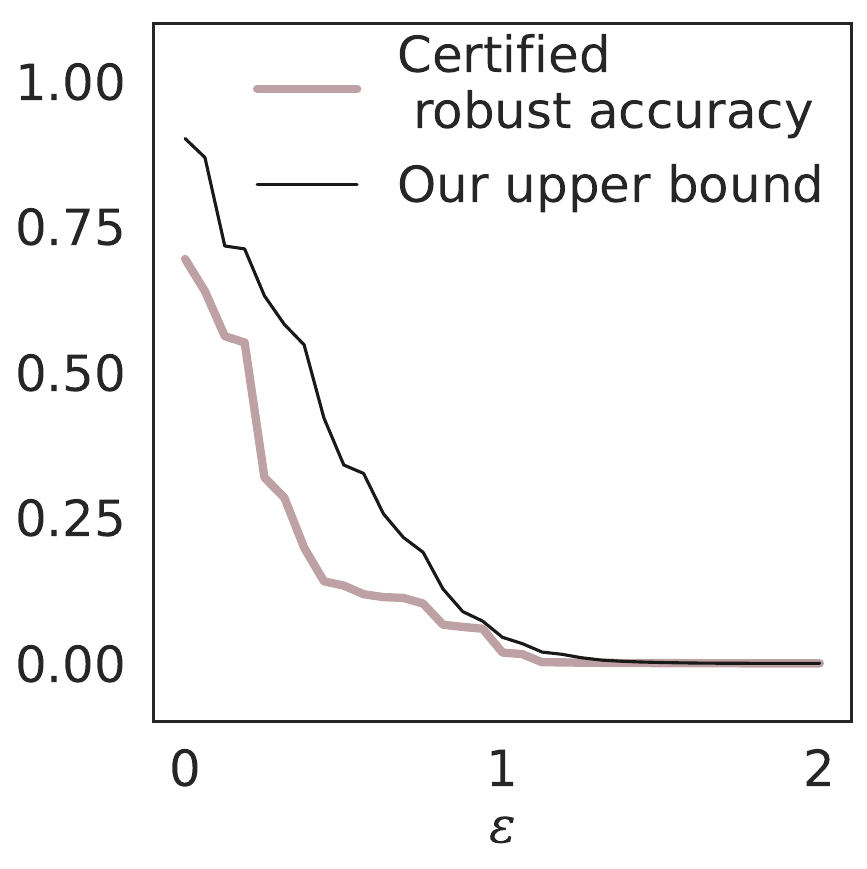}
    \caption{Moons $L^\infty$}
    \end{subfigure}%
    \begin{subfigure}[t]{0.23\linewidth}
    \centering
    \includegraphics[width=\linewidth]{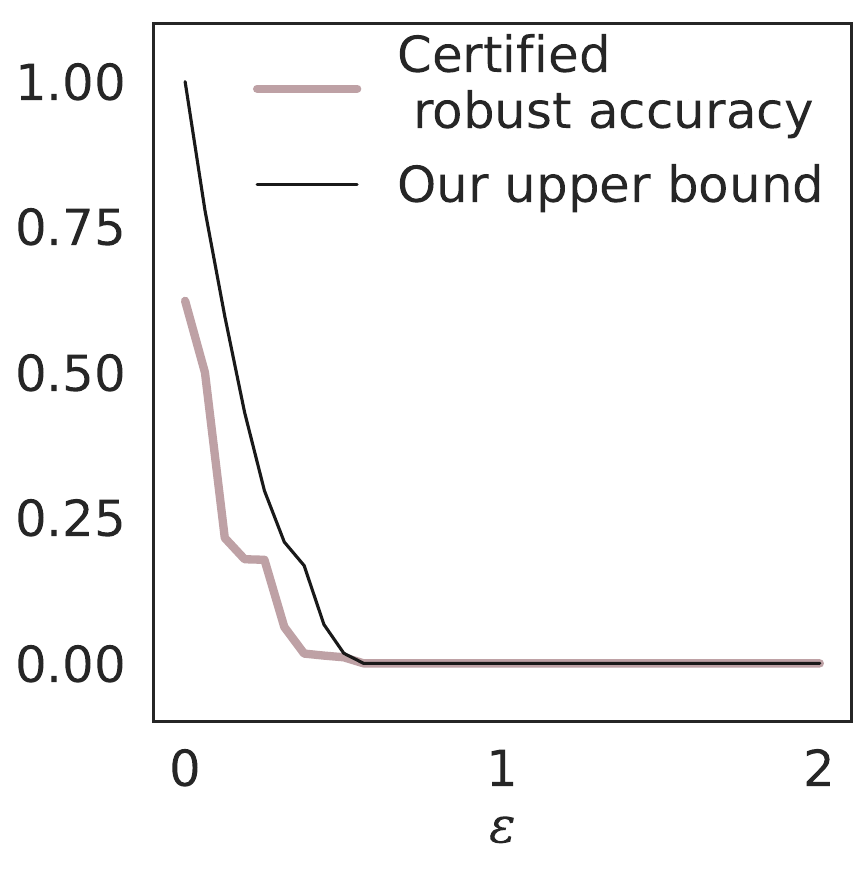}
    \caption{Chan $L^\infty$}
    \end{subfigure}%
    \begin{subfigure}[t]{0.23\linewidth}
    \centering
    \includegraphics[width=\linewidth]{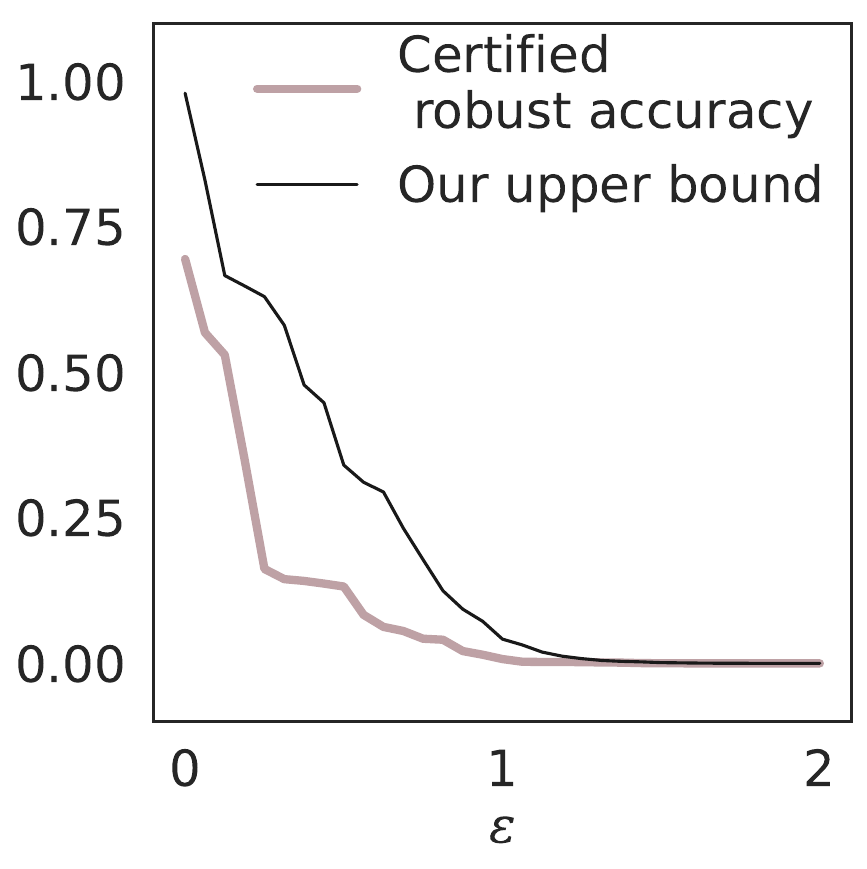}
    \caption{{\scriptsize FashionMNIST} $L^\infty$}
    \end{subfigure}%
    \begin{subfigure}[t]{0.23\linewidth}
    \centering
    \includegraphics[width=\linewidth]{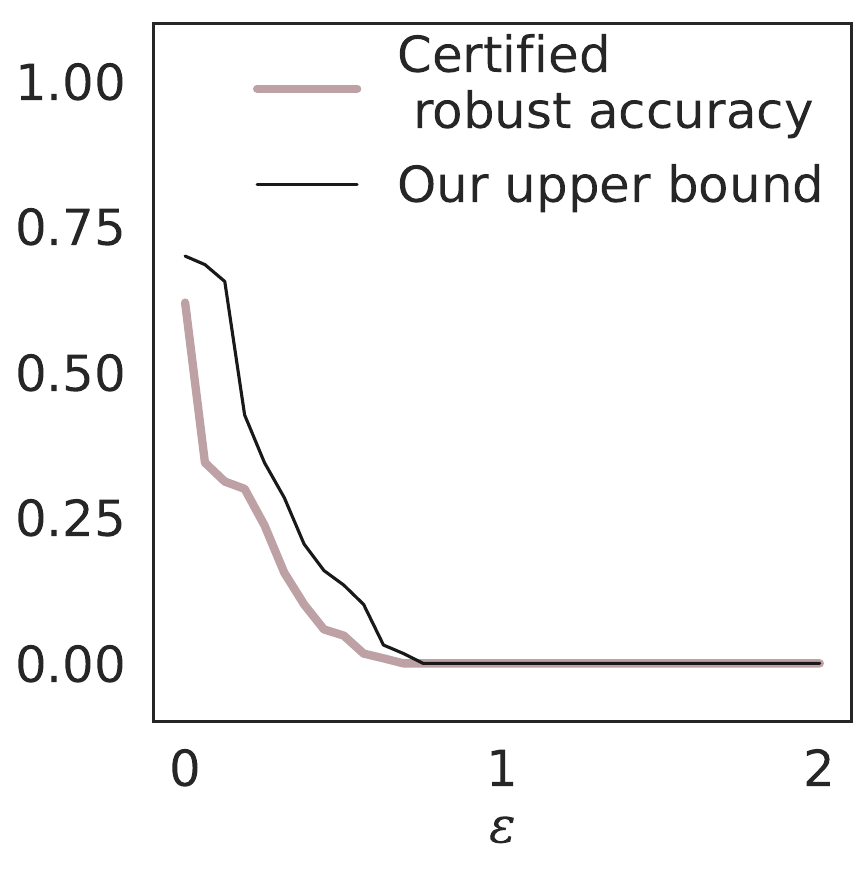}
    \caption{CIFAR-10 $L^\infty$}
    \end{subfigure}\\%
    \begin{subfigure}[t]{0.23\linewidth}
    \centering
    \includegraphics[width=\linewidth]{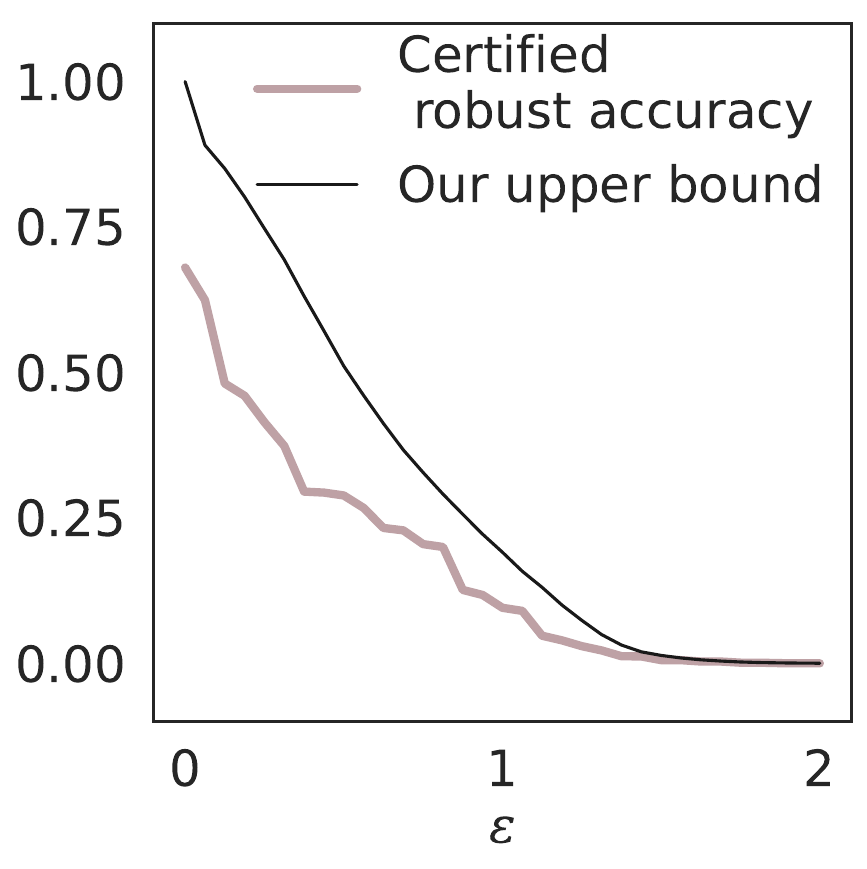}
    \caption{Moons $L^2$}
    \end{subfigure}%
    \begin{subfigure}[t]{0.23\linewidth}
    \centering
    \includegraphics[width=\linewidth]{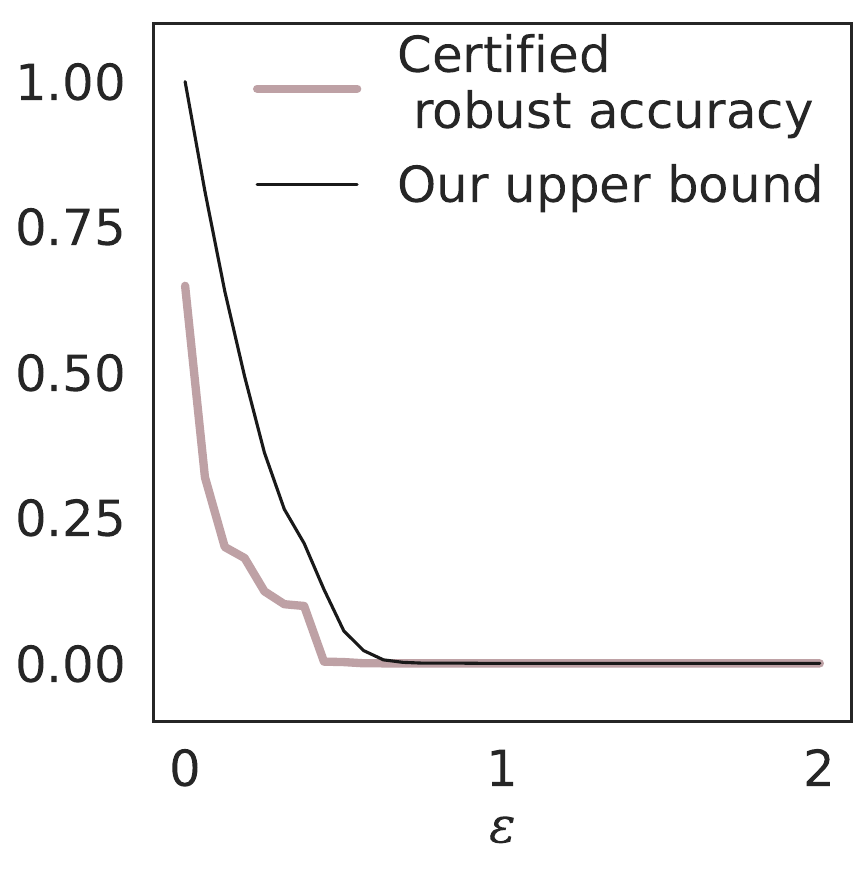}
    \caption{Chan $L^2$}
    \end{subfigure}%
    \begin{subfigure}[t]{0.23\linewidth}
    \centering
    \includegraphics[width=\linewidth]{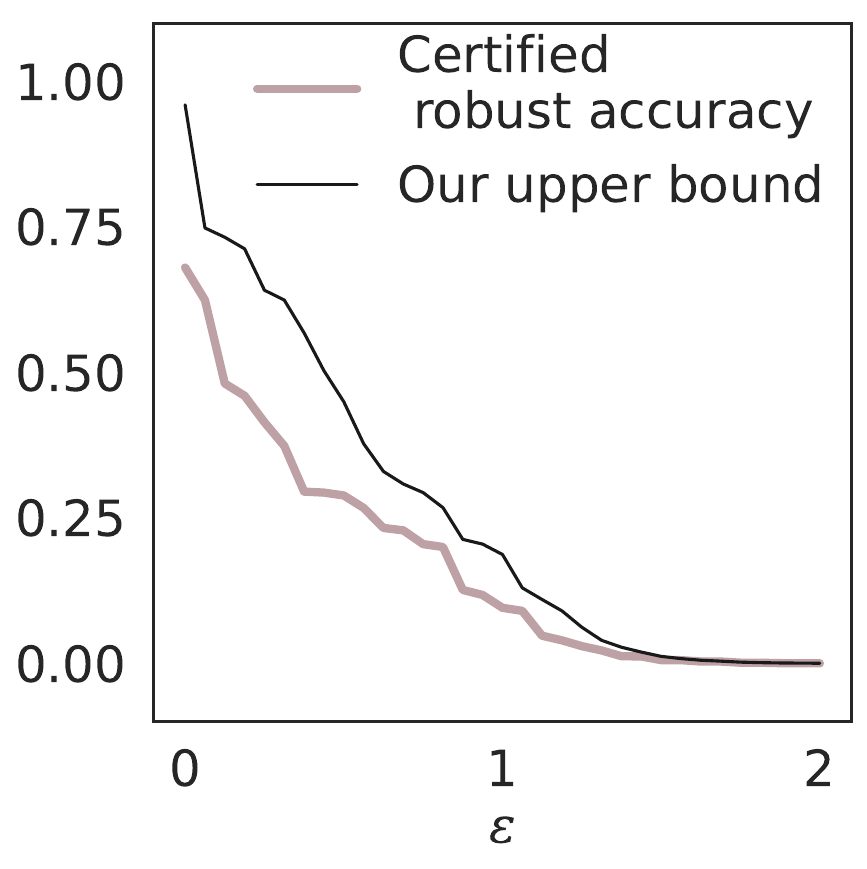}
    \caption{FashionMNIST $L^2$}
    \end{subfigure}%
    \begin{subfigure}[t]{0.23\linewidth}
    \centering
    \includegraphics[width=\linewidth]{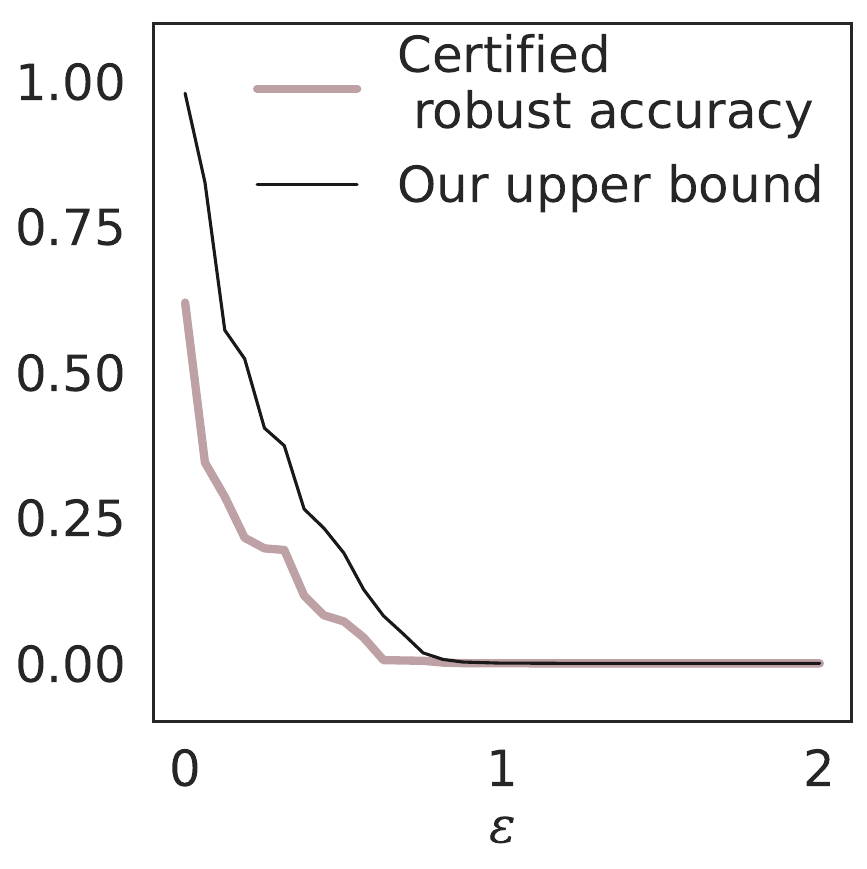}
    \caption{CIFAR-10 $L^2$}
    \end{subfigure}%
    \caption{As epsilon increases, we plot the robustness upper bound change as well as classifiers' certified robust accuracy change in the Moons and Chan dataset.}
    \label{fig:epsilon}
\end{figure}

\noindent \emph{RQ3: How does the upper bound of robustness vary when the vicinity size grows?} In the following, we investigate what can influence the value of irreducible robustness error/upper bound of certified robust accuracy. We already know that when the vicinity grows, it empirically becomes more difficult for a classifier to be robust~\cite{muller2022certified}. The question is then: how about the irreducible robustness error? Is it dependent on the size of the vicinity? If so, how are they correlated? To answer this question, we extend our experiment to cover various vicinity shapes ($L^\infty$ and $L^2$), and different vicinity sizes (from 0 to 2).

The results are shown in \cref{fig:epsilon}. Each sub-figure in \cref{fig:epsilon} illustrates the impact of increasing the vicinity size ($\epsilon$) on the upper bound ($1-\zeta_D$). For instance, in \cref{fig:epsilon}a, we present the change of $1-\zeta_D$ as well as the classifier's certified robust accuracy after certified training. We observe that for all datasets (Moons, Chan, FashionMNIST, CIFAR-10) and norms ($L^\infty$ and $L^2$), as the vicinity size grows, both the robustness upper bound and certified robust accuracy decrease monotonically. This indicates an inverse relationship between the upper bound and vicinity size. This finding aligns with our intuition that when vicinity size grows it becomes more difficult for a classifier to be robust. Notably, the CIFAR-10/Chan dataset shows a sharper decline in the upper bound than Moons/FashionMNIST, suggesting that some data distributions may inherently withstand perturbation better, which is consistent with our previous findings.

\section{Related Works}
This work is closely related to research on Bayes errors and certified training. Computing the Bayes error of a given data distribution has been studied for over half a century~\cite{fukunaga1975k}. Several works have derived upper and lower bounds of the Bayes error and proposed ways to estimate those bounds. Various \( f \)-divergences, such as the Bhattacharyya distance~\cite{fukunaga1990introduction} or the Henze-Penrose divergence~\cite{berisha2015empirically,sekeh2020learning}, have been studied. Other approaches include directly estimating the Bayes error with \( f \)-divergence representation instead of using a bound~\cite{noshad2019learning}, and computing the Bayes error of generative models learned using normalizing flows~\cite{kingma2018glow,theisen2021evaluating}. More recently, a method has been proposed to evaluate Bayes error estimators on real-world datasets~\cite{renggli2021evaluating}, for which we usually do not know the true Bayes error. While these existing studies concentrate on vanilla accuracy, our approach extends the study into the realm of robustness. Besides, some studies may argue that the real-world datasets are well-separated so therefore the Bayes error predicted by the theorems may not be as severe~\cite{yang2020closer}. However, due to the information loss (photo-taking or compression), Bayes errors inevitably exist. For instance, Over 1/3 of CIFAR-10 inputs have been re-annotated by human annotators to have non-fixed labels (CIFAR-10H)~\cite{peterson2019human}, indicating non-zero uncertainty. Hence, calculating irreducible error, regardless of severity, holds significance in understanding the inherent limit of certified robustness.

Many certified training techniques have been developed to increase certified robust accuracy, including branch-and-bound~\cite{bak2020improved,gehr2018ai2,wang2018formal}, linear relaxation~\cite{lyu2021towards,singh2019abstract,baader2024expressivity}, Lipschitz or curvature verification~\cite{li2019preventing,trockman2021orthogonalizing}, and others~\cite{lecuyer2019certified}. In addition, a number of training techniques have been proposed specifically for improving certified robustness~\cite{li2023sok}, which include warm-up training~\cite{shi2021fast}, small boxes~\cite{muller2022certified}, and so on. Certified robust accuracy has seen only limited growth over the past decade, prompting research efforts to understand why. Besides the Bayes error perspective, there exists an explanation for this problem from the standpoint of the abstraction domain~\cite{mirman2021fundamental}. However, note that these studies often only focus on the concept of interval arithmetic. Additionally, factors such as the trade-off between certified robustness and vanilla accuracy have also been explored~\cite{pang2022robustness,tsipras2018robustness}.

\section{Conclusion}
\label{sec:conclusion}

In this work, we study the limit of classification robustness against perturbations. We are motivated by the observation that the robustness of existing certified classifiers tends to be suboptimal, and hypothesise that there is an irreducible robustness error linked to the classification distribution itself. We formally prove that this irreducible robustness error does exist and is greater than the Bayes error. Further, we present how to calculate the upper bound of robustness based on the data distribution and the vicinity within which we demand robustness. Besides, this work also provides empirical experiments that compute our upper bound on common machine learning data sets. Results show that our robustness upper bound is empirically effective. We conclude that the limit of classification robustness can be well elaborated from the Bayes error perspective and we hope that the upper bound we derive can enlighten future developments on certified training and other robust-classifier training.

\subsubsection*{Acknowledgements.}
We thank anonymous reviewers for their constructive feedback.
This research is supported by the Ministry of Education, Singapore under its Academic Research Fund Tier 3 (Award ID: MOET32020-0004). Any opinions, findings and conclusions or recommendations expressed in this material are those of the author(s) and do not reflect the views of the Ministry of Education, Singapore.

\subsubsection{Disclosure of Interests.}

The authors have no competing interests to declare that are relevant to the content of this article.

\appendix

\section{Robustness and Local Consistency}

When we estimate robustness, we need to clarify what are the conditions to be satisfied. As outlined in \cref{def:ae}, a key characteristic of an adversarial example is that it is wrongly predicted, which can reflect that a classifier is not that accurate. Another characteristic is that a benign input is very close to this wrongly predicted input, which then makes this wrongly predicted input hard to detect.

There is actually a third interpretation of robustness, which focuses solely on the local consistency, \emph{i.e.}, true labels are never a factor. This third interpretation is an independent objective for optimisation~\cite{tsipras2018robustness}. A classifier that always predicts a single class is completely consistent, but its accuracy can be as low as $1/\abs{\mathbb{Y}}$. Therefore, we only study robustness from the first and second viewpoint, \emph{i.e.}, robustness as general correctness and consistency within the vicinity, or robustness as consistency within the vicinity with the centre correctly predicted.

In many cases, these two characteristics do not have a clear difference. For example, both ask for the centre input to be correctly predicted. However, when we write their formal expressions, these two characteristics represent very different concepts. The first characteristic is captured in Formula~\ref{eq:correct}, and the second characteristic is captured in Formula~\ref{eq:blind}. In fact, most existing studies adopt only Formula~\ref{eq:blind}, for practical reasons, \emph{e.g.}, the true label of the (possibly adversarial) example may be unknown. In this case, optimise for consistency (alongside accuracy of centre inputs) is the only option.

When it comes to viewing robustness from viewpoint 1, it is more challenging than from viewpoint 2, especially when there is uncertainty in the distribution. As given in \cref{def:robustness}, robustness may be attained at $\bm{x}$ only if every example in its vicinity is classified correctly. Here, we formally explain the third finding in RQ2, as well as the algorithm. Let $n_{\bm{x}}$ denote the number of occurrences of an input $\bm{x}$, then the highest probability that every example in its vicinity is correctly classified is as follows.
\begin{equation}
\label{eq:prod}
    \prod_{\bm{x'}\in\mathbb{V}_{\bm{x}},~~ n_{\bm{x'}}} \max_k  p(\textnormal{y}=k|\mathbf{x}=\bm{x'})
\end{equation}
Then, if in this vicinity there is at least one neighbour $\bm{x''}$ whose label is uncertain, we can define a function $a(\bm{x})$ such that the $a(\bm{x}) = p(\textnormal{y}=k|\bm{x}) < 1$ for $\bm{x} =\bm{x''}$, and $a(\bm{x}) =1$ for others. As such, $\max_k p(\textnormal{y}=k|\bm{x'})\le q(\bm{x'})$, and
\begin{equation}
    \mathrm{Eq.~(\ref{eq:prod})}\le \prod_{\bm{x'}\in\mathbb{V}_{\bm{x}},~~ n_{\bm{x'}}} a(\bm{x'}) \quad=   a(\bm{x''})^{n_{\bm{x''}}}
\end{equation}
Then, when we keep sampling, $n_{\bm{x''}}$ grows, and
\begin{equation}
    \lim_{n_{\bm{x''}}\to\infty} a(\bm{x''})^{n_{\bm{x''}}}=0
\end{equation}
Thus, the general correctness within the vicinity, as part necessary condition for robustness (from viewpoint 1), may be attained at an input only if none of its neighbours, including itself, has uncertainty in their labels.

Moreover, note that the attempt to justify ``consistency within the vicinity with the centre correctly predicted'' as ``general correctness and consistency within the vicinity'' is problematic. Formally, even if a neighbour is predicted the same as its centre, and the centre is correctly predicted, it does not follow that this neighbour is \emph{correctly} predicted. Otherwise, we will get the following consequences. Suppose a classifier achieves consistency at $\bm{x}$
\begin{equation}
\begin{aligned}
    \lnot\exists &(\bm{x}, y), (\bm{x_1}, \bm{y_1}), (\bm{x_2}, \bm{y_2})\in \mathbb{X}\times\mathbb{Y}.~~\\ &(d(\bm{x_1}, \bm{x}) \le \epsilon\land d(\bm{x_2}, \bm{x}) \le \epsilon \land \bm{y_2} \neq \bm{y_1} )
\end{aligned}
\end{equation}
Consider the triangle inequality about distances such that  $(d(\bm{x_1}, \bm{x_2}) \le \epsilon\land d(\bm{x_2}, \bm{x}) \le \epsilon \to d(\bm{x}, \bm{x_2}) \le \epsilon$. Then, we get
\begin{equation}
    \forall (\bm{x_1}, \bm{y_1}), (\bm{x_2}, \bm{y_2})\in \mathbb{X}\times\mathbb{Y}.~~d(\bm{x_1}, \bm{x_2}) \le \epsilon\to \bm{y_1} = \bm{y_2}
\end{equation}
Since $\mathbb{X}$ is a connected domain (for most classification tasks), this local constancy can lead to global constancy, \emph{i.e.}, the labels would be constant across the entire distribution and would be trivial for classification purposes.

\subsection{Algorithms for Computing General Correctness}
We also present the algorithms that realistically compute this general correctness within the vicinity. The \cref{alg:double} uses two samples, $X_{\text{small}}$ and $X_{\text{large}}$, to evaluate the model's robustness. This separation allows for testing the model on a smaller set ($X_{\text{small}}$) against a larger reference set ($X_{\text{large}}$). 

In contrast, \cref{alg:single} employs a single sample $X$, which serves both as the source for generating test inputs and as the reference for distance calculations. This design leads to each input in the data set being compared with all other inputs in the same data set, including itself. A notable implication of this approach is the self-comparison aspect, where each input is guaranteed to be its own neighbour due to zero distance in self-comparison. While this might provide insights into the model's performance on the homogeneous data set, it potentially skews the accuracy assessment, particularly in cases where the neighbourhood sizes are small. Moreover, the use of a single data set in the revised algorithm impacts its applicability in assessing model generalizability. Since the model is tested on the same data used in distance calculations, the evaluation focuses more on internal consistency rather than generalization to new data.

In summary, \cref{alg:double}, with its separate test and reference sets, is more aligned with scenarios requiring assessment of generalizability to new data. In contrast, the \cref{alg:single}'s use of a single data set offers a more introspective look at model performance within a homogeneous data set. Nevertheless, in the experiment, the difference is almost negligible.

\begin{algorithm}
\caption{Testing Process for Moon-Shaped Data Clusters}
\label{alg:double}
\begin{algorithmic}[1]
\REQUIRE $n$: Number of samples in the large data set.
\REQUIRE $\theta$: Distance threshold.
\ENSURE $\alpha$: Proportion of data inputs with correct predictions.

\STATE $X_{\text{small}} \in \mathbb{R}^{10000 \times 2}, X_{\text{large}} \in \mathbb{R}^{n \times 2}, Y_{\text{large}} \in \{0,1\}^n$ (Generate data sets)
\STATE Initialize $C \in \mathbb{R}^{10000}$ and $N \in \mathbb{N}^{10000}$ as zero vectors

\FOR{$i = 1$ to $10000$}
    \STATE $D_i \leftarrow \max_{j=1}^{n} \left| X_{\text{large}, j} - X_{\text{small}, i} \right| $ \COMMENT{Chebyshev distance}
    \STATE $I_i \leftarrow \{j : D_{i,j} \leq \theta\}$
    \STATE $X_{\text{sub}, i} \leftarrow X_{\text{large}, I_i}, Y_{\text{sub}, i} \leftarrow Y_{\text{large}, I_i}$
    \STATE Convert $X_{\text{sub}, i}$ and $Y_{\text{sub}, i}$ to tensors $T_{X,i}$ and $T_{Y,i}$
    \STATE Move $T_{X,i}$ and $T_{Y,i}$ to GPU if available
    \STATE $O_i \leftarrow \text{model}(T_{X,i})$ \COMMENT{Model predictions without gradient tracking}
    \STATE $C_i \leftarrow \sum \mathbb{I}(\text{round}(O_i) == T_{Y,i})$ \COMMENT{Count correct predictions}
    \STATE $N_i \leftarrow |I_i|$ \COMMENT{Number of neighbors}
\ENDFOR

\STATE $\alpha \leftarrow \frac{1}{10000} \sum_{i=1}^{10000} \mathbb{I}(C_i == N_i)$ \COMMENT{Proportion of perfect accuracies}
\PRINT $\text{'Proportion of accuracies equal to 1: '}, \alpha$

\end{algorithmic}
\end{algorithm}

\begin{algorithm}
\caption{Testing Process for The same group of data}
\label{alg:single}
\begin{algorithmic}[1]
\REQUIRE $n$: Number of samples in the large data set.
\REQUIRE $\theta$: Distance threshold.
\ENSURE $\alpha$: Proportion of data inputs with correct predictions.

\STATE $ X \in \mathbb{R}^{n \times 2}, Y \in \{0,1\}^n$ (Generate data sets)
\STATE Initialize $C \in \mathbb{R}^{10000}$ and $N \in \mathbb{N}^{10000}$ as zero vectors

\FOR{$i = 1$ to $10000$}
    \STATE $D_i \leftarrow \max_{j=1}^{n} \left| X_{ j} - X_{i} \right| $ \COMMENT{Chebyshev distance}
    \STATE $I_i \leftarrow \{j : D_{i,j} \leq \theta\}$
    \STATE $X_{\text{sub}, i} \leftarrow X_{ I_i}, Y_{\text{sub}, i} \leftarrow Y_{ I_i}$
    \STATE Convert $X_{\text{sub}, i}$ and $Y_{\text{sub}, i}$ to tensors $T_{X,i}$ and $T_{Y,i}$
    \STATE Move $T_{X,i}$ and $T_{Y,i}$ to GPU if available
    \STATE $O_i \leftarrow \text{model}(T_{X,i})$ \COMMENT{Model predictions without gradient tracking}
    \STATE $C_i \leftarrow \sum \mathbb{I}(\text{round}(O_i) == T_{Y,i})$ \COMMENT{Count correct predictions}
    \STATE $N_i \leftarrow |I_i|$ \COMMENT{Number of neighbors}
\ENDFOR

\STATE $\alpha \leftarrow \frac{1}{10000} \sum_{i=1}^{10000} \mathbb{I}(C_i == N_i)$ \COMMENT{Proportion of perfect accuracies}
\PRINT $\text{'Proportion of accuracies equal to 1: '}, \alpha$

\end{algorithmic}
\end{algorithm}

\section{Robustness and Convolution}

Local consistency always plays a role in robustness, no matter whether general correctness is required. The requirement of local consistency can be viewed as each input's label must be somehow linked to its neighbours' labels. In \cref{sec:conv}, we present how labels are assigned in the context of robustness. Here, we show why this label-assignment action, when applied globally to the distribution, is equivalent to convolving the distribution with the vicinity function.

\begin{equation}
\begin{aligned}
    \operatorname{E}_{(\mathbf{x},\textnormal{y})\sim D} &\left[\lfloor\int_\mathbb{X} v (\mathbf{x} - \bm{x'})\cdot\bm{1}_{ \textnormal{y} = h(\bm{x'}) } d \bm{x'}\rfloor\right]\\
    = \sum_k\int &\lfloor\int_\mathbb{X} v (\bm{x} - \bm{x'})\cdot\bm{1}_{ k = h(\bm{x'}) } d \bm{x'}\rfloor p(\bm{x},k) d \bm{x}\\
    = \int \sum_k&\lfloor\int_\mathbb{X} v (\bm{x} - \bm{x'})\cdot\bm{1}_{ k = h(\bm{x'}) } d \bm{x'}\rfloor p(\bm{x},k) d \bm{x}\\
    = \int \sum_k&\lfloor\int_\mathbb{X} v (\bm{x} - \bm{x'})\cdot g_k(\bm{x'}) d \bm{x'}\rfloor p(\bm{x},k) d \bm{x}
\end{aligned}
\end{equation}
where
\begin{equation}
    g_k(\bm{x}) = \begin{cases}
        1, &\text{if}~ h(\bm{x}) = k\\
        0, &\text{otherwise}
    \end{cases}
\end{equation}
Also, we can understand that for any $\bm{x}$,
\begin{equation}
    \begin{aligned}
        g_k(\bm{x}) \in \set{0, 1}\\
        \sum_k g_k(\bm{x}) = 1
    \end{aligned}    
\end{equation}
Since the vicinity function is an even function, we further have
\begin{equation}
\begin{aligned}
     \operatorname{E}_{(\mathbf{x},\textnormal{y})\sim D} &\left[\lfloor\int_\mathbb{X} v (\mathbf{x} - \bm{x'})\cdot\bm{1}_{ \textnormal{y} = h(\bm{x'}) } d \bm{x'}\rfloor\right]\\   = \int \sum_k&\lfloor\int_\mathbb{X} v (\bm{x'} - \bm{x})\cdot g_k(\bm{x'}) d \bm{x'}\rfloor p(\bm{x},k) d \bm{x}\\
\end{aligned}
\end{equation}

\section{Interpretation of \emph{Original} Distribution}

The Bayes error reflects the inherent uncertainty in a given distribution and an increased Bayes error essentially represents a growth of uncertainty. We are interested in how the uncertainty can come into being in the original distribution, sometimes referred to as \emph{a priori} distribution.

\begin{theorem}[Conditional Existence of Distinct Outputs for Identical Inputs]
    Given a distribution $D$ over $\mathbb{X}\times\mathbb{Y}$, its Bayes error is $\beta_D$ is greater than zero if and only if there exists at least one pair of examples, denoted by $(\bm{x}, y)$ and $(\bm{x'}, y')$, within the system such that the input inputs of both instances are identical and label of these examples are different, \emph{i.e.},
    \begin{equation}
        \beta_D > 0 \Leftrightarrow \exists (\bm{x}, y), (\bm{x'}, y').~ (\bm{x} = \bm{x'} \land y \neq y')
    \end{equation}
\end{theorem}

\begin{proof}
    We start from an assumption that $\beta_D > 0$ and for every input vector $\bm{x}$ there is a unique class label $y$ such that $p(\textnormal{y}|\mathbf{x}=\bm{x})$ is deterministic. Specifically, for a given $\bm{x}$, one of the class probabilities $p(\textnormal{y}=k|\mathbf{x}=\bm{x})$ is 1, while the others are 0. Refer to the computation of the Bayes error in \cref{eq:bayeserror1}. Since $1 - \max_k p(\textnormal{y}=k|\mathbf{x}=\bm{x})$ equal to 0, the integral evaluates to 0, which contradicts the assumption of a non-zero $\beta_D$. Conversely, if $\beta_D = 0$, then at each $\bm{x}$, either $p(\textnormal{y}=k|\mathbf{x}=\bm{x}))$ or $p(\mathbf{x}=\bm{x})$ has to be zero. This indicates that every \emph{existing} $\bm{x}$ is associated with a unique $y$. 
\qed
\end{proof}

Thus, if a given data set does not have a case where the same input is mapped to two different labels, then theoretically, its Bayes error should be zero. It is crucial to explicitly define the original distribution before we proceed with the calculation of its Bayes error. Additionally, note that the classifiers are expected to learn solely from the training data provided. Should there be any additional data used in the learning process, it must be considered as part of the \emph{a priori} distribution. This requirement is not intended to exclude relevant external data; rather, it is an essential prerequisite for computing the Bayes error.

\section{Experiment Setup Details}
\begin{figure}[t]
  \centering
  \begin{subfigure}[t]{0.23\linewidth}
    \centering
    \includegraphics[width=\linewidth]{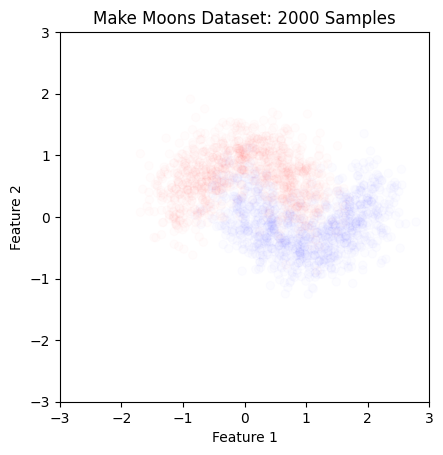}
    \caption{Moons}
  \end{subfigure}%
  \hfill
  \begin{subfigure}[t]{0.23\linewidth}
    \centering
    \includegraphics[width=\linewidth]{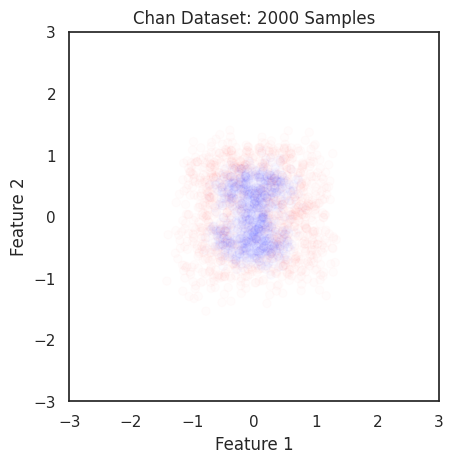}
    \caption{Chan}
  \end{subfigure}%
  \hfill
  \begin{subfigure}[t]{0.23\linewidth}
    \centering
    \includegraphics[width=\linewidth]{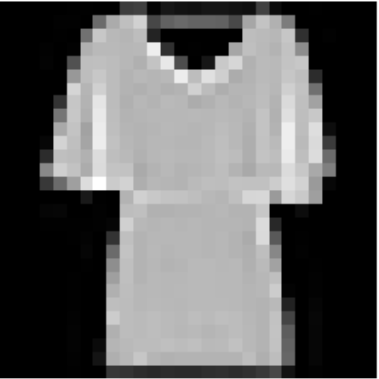}
    \caption{FashionMNIST}
  \end{subfigure}%
  \hfill
  \begin{subfigure}[t]{0.23\linewidth}
    \centering
    \includegraphics[width=0.93\linewidth]{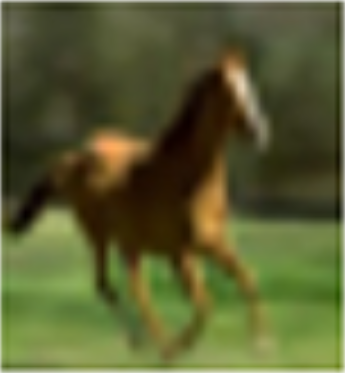}
    \caption{CIFAR-10}
  \end{subfigure}
  \caption{Data sets for classification. We illustrate some examples from each distribution
  \label{fig:dataex}}
\end{figure}

\subsection{Data sets}
Four data sets for classification are used throughout the experiment. They are Moons, Chan, FashionMNIST, and CIFAR-10. Moons and Chan are binary classification data sets while FashionMNIST, and CIFAR-10 are 10-class categorical classification data sets. Some examples of these data sets are given in \cref{fig:dataex}.

\paragraph{Moons}

The Moons data set is for binary classification, \emph{i.e.,} there are two prior probabilities. Each input in the $\mathbb{X}$ domain has features in two dimensions. Like our running example, the distribution of each class can be expressed analytically. For the upper moon, $\textnormal{y} = 0$, the likelihood $p(\mathbf{x}=(x_1, x_2)|\textnormal{y}=0)$ is given as
\begin{equation}
    \frac{1}{2\pi^2\sigma^2} \exp\left(-\frac{x_1^2 + x_2^2 + 1}{2\sigma^2}\right) \int_0^{\pi} \exp\left(\frac{x_1\cos(t) + x_2\sin(t)}{\sigma^2}\right) dt
\end{equation}
And for the lower moon, $\textnormal{y}=1$, the likelihood $p(\mathbf{x}=(x_1, x_2)|\textnormal{y}=1)$ is given as
\begin{equation}
\begin{split}
     \frac{1}{2\pi^2\sigma^2} &\exp\left(-\frac{(x_1 - 1)^2 + (x_2 - 0.5)^2 + 1}{2\sigma^2}\right)\times\\ \int_0^{\pi} &\exp\left(\frac{-(x_1 - 1)\cos(t) - (x_2 - 0.5)\sin(t)}{\sigma^2}\right) dt   
\end{split}
\end{equation}

The density functions of the two moons are derived from the convolution of uniform distributions along the two distinct curves with a Gaussian distribution. The uniform distributions are defined along these curves, while the Gaussian random variable with zero mean and standard deviation $\sigma$ in both $x_1$ and $x_2$ directions. The convolution, representing the combined effect of the uniform and Gaussian distributions, is expressed as a line integral. Numerical integration methods provide an approximation of the integral for specific values of $x_1$ and $x_2$.

Besides the original distribution, we can also analytically express the convolved distribution (\cref{fig:convolved}b) of Moons. Given a function $f(x_1, x_2)$ and a vicinity function $v(x_1, x_2)$, the convolution $(p*v)(x_1, x_2)$ is defined as:
\begin{equation}
    (p*v)(x_1, x_2) = \int_{-\infty}^{\infty} \int_{-\infty}^{\infty} p(u_1, u_2) \cdot v(x_1 - u_1, x_2 - u_2) \, du_1 \, du_2
\end{equation}
Where we substitute $p(x_1, x_2)$ by either $p(\mathbf{x}=(x_1, x_2)|\textnormal{y}=0)$ or $p(\mathbf{x}=(x_1, x_2)|\textnormal{y}=1)$. And, $v(x_1, x_2)$ is the vicinity function (rectangular function) as follows.
\begin{equation}
    v(x_1, x_2) = \begin{cases} 
    \frac{1}{4\epsilon}, & \text{if } -\epsilon \leq x_1 \leq \epsilon \text{ and } -\epsilon \leq x_2 \leq \epsilon \\
    0, & \text{otherwise} 
    \end{cases}
\end{equation}
Substituting these into the convolution integral, and considering that $v(x_1, x_2)$ is non-zero only in a specific range, we have
\begin{equation}
\begin{aligned}
    (p(\mathbf{x}&|\textnormal{y}=0)*v)(x_1, x_2)=\\ & \frac{1}{4\epsilon^2} \cdot \frac{1}{2\pi^2\sigma^2} \int_{x_1-\epsilon}^{x_1+\epsilon} \int_{x_2-\epsilon}^{x_2+\epsilon} \exp\left(-\frac{u_1^2 + u_2^2 + 1}{2\sigma^2}\right)\\&\quad \left[ \int_0^{\pi} \exp\left(\frac{u_1\cos(t) + u_2\sin(t)}{\sigma^2}\right) dt \right] du_1 \, du_2 \\
    (p(\mathbf{x}&|\textnormal{y}=1)*v)(x_1, x_2)=\\ & \frac{1}{4\epsilon^2} \cdot \frac{1}{2\pi^2\sigma^2}\int_{x_1-\epsilon}^{x_1+\epsilon} \int_{x_2-\epsilon}^{x_2+\epsilon} \exp\left(-\frac{(u_1 - 1)^2 + (u_2 - 0.5)^2 + 1}{2\sigma^2}\right)\\&\quad \left[ \int_0^{\pi} \exp\left(\frac{-(u_1 - 1)\cos(t) - (u_2 - 0.5)\sin(t)}{\sigma^2}\right) dt \right] du_1 \, du_2 
\end{aligned}
\end{equation}
This expression represents a triple integral. To solve it numerically, we simplify the integrands and use numerical integration techniques such as \texttt{tplquad} in Python's \texttt{scipy.integrate} module.

\paragraph{Chan} This is a binary classification data set~\cite{chen2023evaluating}, whose distribution function is numerically depicted in \cref{fig:convolved}c. We have $P(\textnormal{y}=0) = P(\textnormal{y}=1) = 1/2$. Like Moons, the input of this data set is in two dimensions, and thus we use the same neural network as the classifier.

\subsection{Classifier Performance Metrics}
We now provide the typical existing performance evaluation metric for classifiers. Given a classifier and test data set, we can compute its standard accuracy and certified robust accuracy as in \cref{tab:metrics}. They essentially represent how we evaluate the standard performance and robustness, respectively, of a classifier.

\begin{table}[t]
\centering
\caption{To evaluate a model $h$ on any input data from some finite data set $D_\mathrm{f}$, we assume $|D_\mathrm{f}|$ is the number of testing samples.
}
\label{tab:metrics}
\footnotesize
\begin{tabular}{p{0.2\linewidth}p{0.38\linewidth}p{0.38\linewidth}}
\toprule
Metric & Formula & Meaning \\ \midrule
Vanilla Accuracy & $\frac{1}{|D_\mathrm{f}|}\sum_{(\bm{x}, y)\in D_\mathrm{f}}$   $\Big(\bm{1}_{h(\bm{x})= y}\Big) $ & The probability that the model's prediction is correct for an input from the data distribution $D$. \\ \midrule

Certified Robust Accuracy & $\frac{1}{|D_\mathrm{f}|}\sum_{(\bm{x}, y)\in D_\mathrm{f}} $ $ \quad\left(\bm{1}_{\operatorname{Vrob}\big(h, \bm{x}, y; \mathbb{V}_{\bm{x}}\big)}\right)$ & The probability that the classifier can be certified robust (consistency and a correct centre prediction), for an input from the data distribution $D$. \\ \midrule

\rowcolor{mygray}\multicolumn{3}{l}{$\operatorname{\mathbb{I}}(\phi)$ a function that returns 1 if $\phi$ is satisfied and 0 otherwise}\\\bottomrule
\end{tabular}
\end{table}

\subsection{Tools, Packages, and Hardware}
The computational tools required for our method are detailed here. We employ numerical quadrature integration, \emph{e.g.}, using \texttt{dblquad} from \texttt{scipy.integrate} for handling double integration tasks. For convolution processes, we utilize FFT-based convolution, implemented through \texttt{scipy. signal. fftconvolve}. Additionally, all neural networks in our study are developed using PyTorch. The training and testing of these networks are conducted on a Tesla T4 GPU.

\end{document}